\pdfoutput=1
\documentclass{article}
\usepackage{microtype}
\usepackage{graphicx}
\usepackage{subfigure}
\usepackage{booktabs} % for professional tables

\usepackage{hyperref}

\usepackage{amsmath,amsfonts,bm}

\usepackage{xcolor}
\usepackage{colortbl}

\def\eqref#1{equation~\ref{#1}}
\def\ceil#1{\lceil #1 \rceil}
\def\floor#1{\lfloor #1 \rfloor}
\def\1{\bm{1}}

\def\rx{{\textnormal{x}}}
\def\ry{{\textnormal{y}}}

\def\vx{{\bm{x}}}
\def\vy{{\bm{y}}}

\def\mA{{\bm{A}}}

\def\mR{{\bm{R}}}

\def\mX{{\bm{X}}}
\def\mY{{\bm{Y}}}

\DeclareMathAlphabet{\mathsfit}{\encodingdefault}{\sfdefault}{m}{sl}
\SetMathAlphabet{\mathsfit}{bold}{\encodingdefault}{\sfdefault}{bx}{n}

\def\sR{{\mathbb{R}}}

\usepackage{url}
\usepackage{multirow}
\usepackage{paralist}

\usepackage{booktabs} % for professional tables
\usepackage{multicol}

\usepackage[ruled,noend]{algorithm2e}

\SetCommentSty{mycommfont}

\usepackage{here}
\usepackage{amsmath,amssymb,amsfonts,amsbsy,amsfonts,latexsym}
\usepackage{amsthm}
\usepackage{multirow}
\usepackage{makecell}
\usepackage{xcolor}
\usepackage{colortbl}

\definecolor{colorA}{RGB}{189,201,225}
\definecolor{colorB}{RGB}{103,169,207}
\definecolor{colorC}{RGB}{ 28,144,153}
\definecolor{colorD}{RGB}{  1,108, 89}

\newcolumntype{R}{>{\columncolor{gray!40}}r}
\newcolumntype{L}{>{\columncolor{gray!40}}l}
\newcolumntype{C}{>{\columncolor{gray!40}}c}

\usepackage{tabularx,colortbl,xcolor}
\usepackage{multirow}
\usepackage[normalem]{ulem}
\useunder{\uline}{\ul}{}

\usepackage{enumitem}

\usepackage{xparse}% http://ctan.org/pkg/xparse

\graphicspath{{fig}}
\DeclareGraphicsExtensions{.pdf,.png}
\usepackage{pgfplots, pgfplotstable}

\SetKwInput{KwInput}{Input}
\SetKwInput{KwOutput}{Output}

\NewDocumentCommand{\var}{O{s} m O{}}{%
  \ensuremath{#1_{#2}^{#3}}% add \vphantom{<bizarre sup>}
}
\usepackage{siunitx}

\definecolor{light-gray}{gray}{0.80}

\renewcommand\paragraph{\subsubsection*}

\newcommand\aref{Algorithm \ref}
\newcommand\eref{Eq.~\ref}
\newcommand\fref{Figure ~\ref}
\newcommand\tref{Table~\ref}
\newcommand\sref{Section~\ref}

\def\x{{\bf x}}

\def\0{{\bf 0}}

\usepackage{amsthm}

\newtheorem{mytheorem}{Theorem}
\newtheorem{definition}[mytheorem]{Definition}
\newtheorem{assumption}[mytheorem]{Assumption}
\newtheorem{lemma}[mytheorem]{Lemma}
\newtheorem{fact}[mytheorem]{Fact}

\usepackage{mathabx} % for '\widecheck' macro
\def\xhat{\widehat \mX}
\def\xcheck{\widecheck \mX}
\def\xtilde{\widetilde \mX}
\def\psib{{\bm{\psi_{B}}}}
\def\mub{\bm{\mu}_{B}}
\def\sigmab{\bm{\sigma}_{B}}

\def\psib{{{\psi_{B}}}}
\def\mub{{\mu}_{B}}
\def\sigmab{{\sigma}_{B}}

\def\xi{\vx_i}
\def\xj{\vx_j}
\def\yi{\vy_i}

\def\xhati{\hat \x_i}
\def\xchecki{\check \x_i}
\def\xhatj{\hat \x_j}
\def\xcheckj{\check \x_j}

\def\loss{\mathcal{L}}
\def\losshat{\widehat{\mathcal{L}}}

\def\1{\vec{1}}
\graphicspath{{fig/}}

\usepackage[accepted]{icml2020}

\newcommand{\OURS}{\textsc{PN}\xspace}
\newcommand{\OURSV}{\textsc{PN-V}\xspace}

\newcommand{\transbn}{Transformer$_{\textsc{BN}}$\xspace}
\newcommand{\transln}{Transformer$_{\textsc{LN}}$\xspace}
\newcommand{\transpn}{Transformer$_{\textsc{PN}}$\xspace}
\newcommand{\transpnv}{Transformer$_{\textsc{PN-V}}$\xspace}
\newcommand{\sm}{${\texttt{small}}$\xspace}
\newcommand{\bg}{${\texttt{big}}$\xspace}

\newcommand{\BN}{BN\xspace}
\newcommand{\bn}{BN\xspace}

\usepackage{amssymb}

\begin{document}

\twocolumn[
\icmltitle{PowerNorm: Rethinking Batch Normalization in Transformers}
\icmlsetsymbol{equal}{*}

\begin{icmlauthorlist}
\icmlauthor{Sheng Shen}{equal,berkeley}
\icmlauthor{Zhewei Yao}{equal,berkeley}
\icmlauthor{Amir Gholami}{berkeley}
\icmlauthor{Michael W. Mahoney}{berkeley}
\icmlauthor{Kurt Keutzer}{berkeley}
\end{icmlauthorlist}
\icmlaffiliation{berkeley}{UC Berkeley}

% \icmlcorrespondingauthor{Kurt Keutzer}{keutzer@berkeley.edu}
\icmlcorrespondingauthor{Amir Gholami}{amirgh@berkeley.edu}

\icmlkeywords{Machine Learning, Natural Language Processing, Transformers}

\vskip 0.3in
]



\printAffiliationsAndNotice{\icmlEqualContribution} % otherwise use the standard text.

\begin{abstract}

The standard normalization method for neural network (NN) models used in Natural Language Processing (NLP) is layer normalization (LN).
This is different than batch normalization (BN), which is widely-adopted in Computer Vision. 
The preferred use of LN in NLP is principally due to the empirical observation that a (naive/vanilla) use of BN leads to significant performance degradation for NLP tasks; however, a thorough understanding of the underlying reasons for this is not always evident.
In this paper, we perform a systematic study of NLP transformer models to understand why BN has a poor performance, as compared to LN.
We find that the statistics of NLP data across the batch dimension exhibit large fluctuations throughout training.
This results in instability, if BN is naively implemented.
To address this, we propose Power Normalization (\OURS), a novel normalization scheme
that resolves this issue by
(i) relaxing zero-mean normalization in BN,
(ii) incorporating a running quadratic mean instead of per batch statistics to stabilize
fluctuations,
and (iii) using an approximate backpropagation for incorporating the running statistics in the forward pass.
We show theoretically, under mild assumptions, that \OURS leads to a smaller Lipschitz constant for the loss, compared with BN. 
Furthermore, we prove that the approximate backpropagation scheme leads to bounded gradients.
We extensively test \OURS for transformers on a range of NLP tasks, and we show that
it significantly outperforms both LN and BN. 
In particular, \OURS outperforms LN by 0.4/0.6 BLEU on IWSLT14/WMT14 and 5.6/3.0 PPL on PTB/WikiText-103. We make our code publicly available at \url{https://github.com/sIncerass/powernorm}.

\end{abstract}
\section{Introduction}

Normalization has become one of the critical components in Neural Network (NN) architectures
for various machine learning tasks, in particular
in Computer Vision (CV) and Natural Language Processing (NLP). 
However, currently there are very different forms of normalization used in CV and NLP. 
For example, Batch Normalization (BN)~\cite{ioffe2015batch} is widely adopted in CV, but it leads to significant performance degradation when naively used in NLP. 
Instead, Layer Normalization (LN)~\cite{ba2016layer} is the standard normalization scheme used in NLP.
All recent NLP architectures, including Transformers~\cite{vaswani2017attention},
have incorporated LN instead of BN as their default normalization scheme.
In spite of this, the reasons why BN fails for NLP have not been clarified, and a better alternative to LN has not been presented.

In this work, we perform a systematic study of the challenges associated with BN for NLP, and based on this we propose Power Normalization (\OURS), a novel normalization method that significantly outperforms LN. 
In particular, our contributions are as follows:

\begin{itemize}
    \item 
    We find that there are clear differences in the batch statistics of NLP data versus CV data.
    In particular, we observe that batch statistics for NLP data have a very large variance throughout training. This variance exists in the corresponding gradients as well. In contrast, CV data exhibits orders of magnitude smaller variance. 
    See~\fref{fig:cvnlp-fwd-bn} and~\ref{fig:cvnlp-bwd-bn} for a comparison of BN in CV and NLP. 
    
    \item 
    To reduce the variation of batch statistics, we modify typical BN
    by relaxing zero-mean normalization, and we replace the variance with the quadratic mean.
    We denote this scheme as \OURSV.
    We show theoretically that \OURSV preserves the first-order smoothness property as in BN; see Lemma~\ref{lemma:lipschitz_constant_of_pnv}. 
    
    \item 
    We show that using running statistics for the quadratic mean results in significantly better performance, up to 1.5/2.0 BLEU on IWSLT14/WMT14 and 7.7/3.4 PPL on PTB/WikiText-103, as compared to BN; see~\tref{table:translation} and~\ref{table:language_model}.
    We denote this scheme as \OURS. Using running statistics requires correcting the typical backpropagation
    scheme in BN. As an alternative, we propose an approximate backpropagation to capture the running statistics.
    We show theoretically that this approximate backpropagation leads to bounded gradients, which is
    a necessary condition for convergence; see Theorem~\ref{thm:gradient_xtilde_bound}.

    \item We perform extensive tests showing that 
    \OURS also improves performance on machine translation and language modeling tasks, as compared to LN.
    In particular, \OURS outperforms LN by 0.4/0.6 BLEU on IWSLT14/WMT14, and by 5.6/3.0 PPL on PTB/WikiText-103.
    We emphasize that the improvement of \OURS over LN is without any change of hyper-parameters.
    
    \item We analyze the behaviour of \OURS and LN by computing the Singular Value
    Decomposition of the resulting embedding layers, and we show that \OURS leads to a more well-conditioned
    embedding layer; see~\fref{figure:ln-pn-emb-svd}.
    Furthermore, we show that \OURS is robust to small-batch statistics, and it still achieves higher performance,
    as opposed to LN; see~\fref{figure:ablation}. 
\end{itemize}

\begin{figure}[!htp]
\begin{center}
  \includegraphics[width=.48\linewidth, trim={1.5cm 1.5cm 11.5cm 1.5cm},clip]{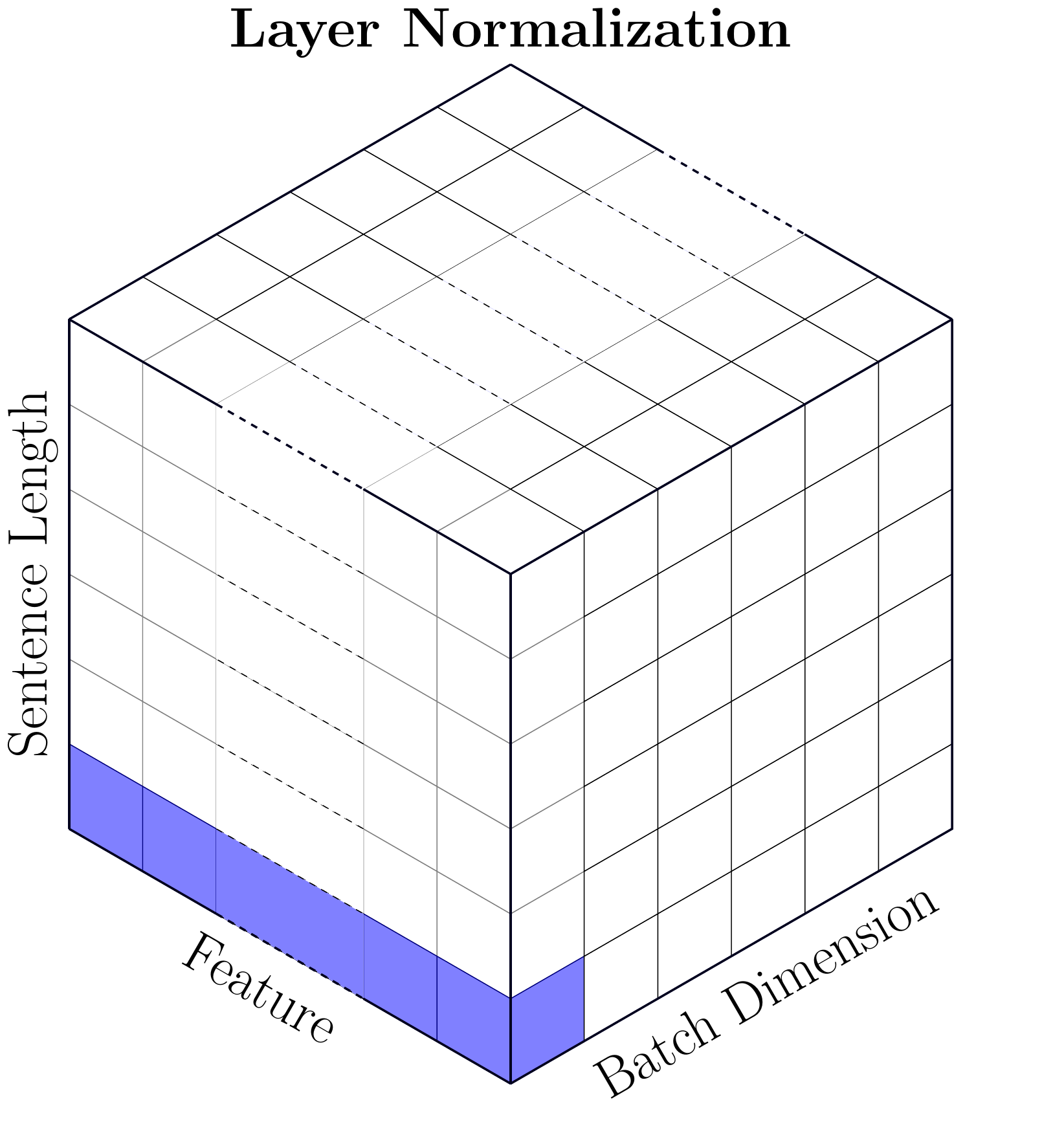}
  \includegraphics[width=.48\linewidth, trim={1.5cm 1.5cm 11.5cm 1.5cm},clip]{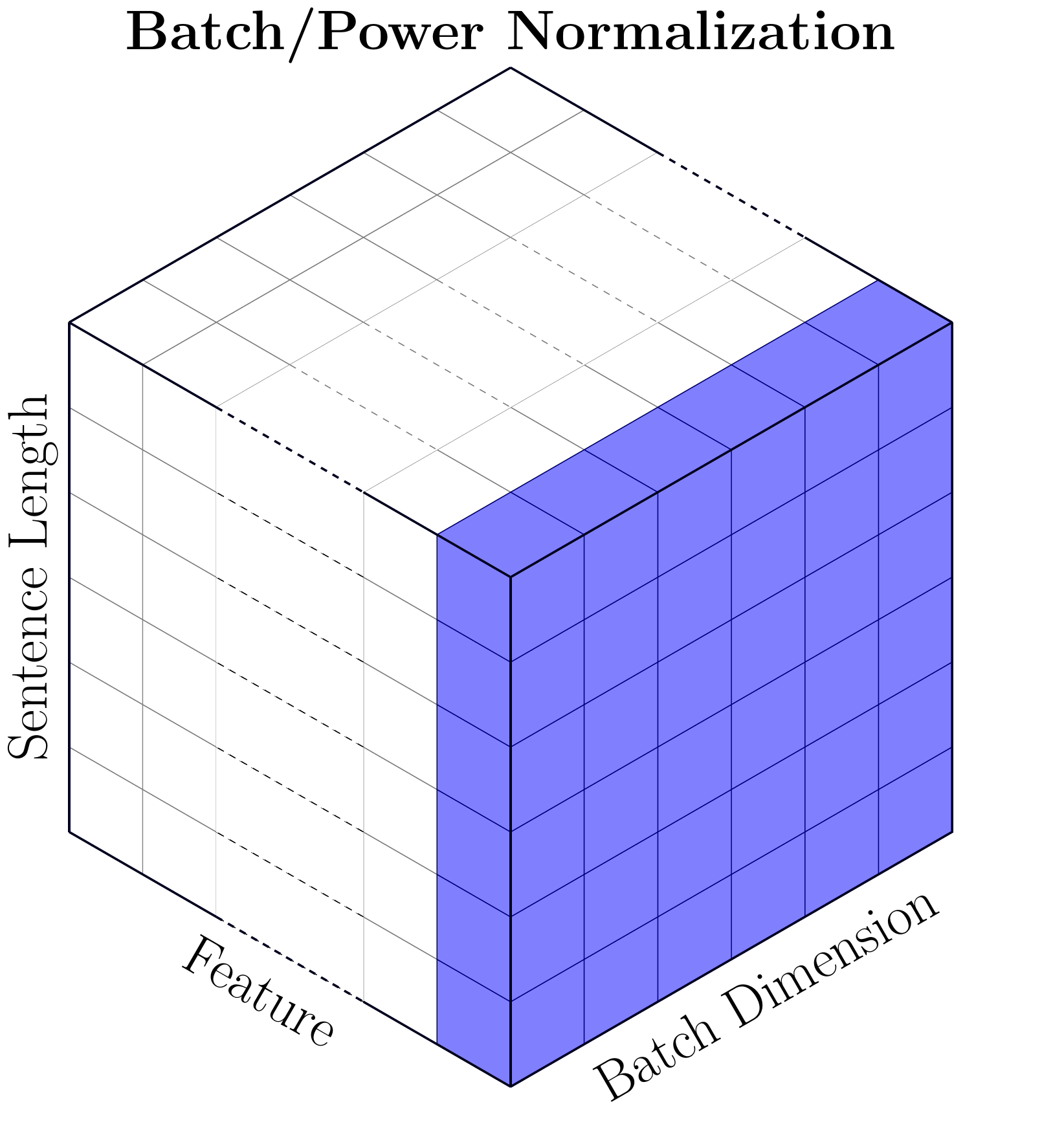}
\end{center}
 \caption{\footnotesize The illustration of layer normalization (left) and batch/power normalization (right). 
 The entries colored in blue show the components used for calculating the
 statistics.
 }
\label{fig:ln-pn-visualition}
\end{figure}

\section{Related Work}

Normalization is widely used in modern deep NNs such as ResNet~\cite{he2016deep}, MobileNet-V2~\cite{sandler2018mobilenetv2}, and DenseNet~\cite{huang2017densely} in CV, as well as LSTMs~\cite{hochreiter1997long,ba2016layer}, transformers~\cite{vaswani2017attention}, and transformer-based models~\cite{devlin2018bert,liu2019roberta} in NLP. 
There are two main categories of normalization: weight normalization~\cite{salimans2016weight,miyato2018spectral,qiao2019weight} and activation normalization~\cite{ioffe2015batch,jarrett2009best,krizhevsky2012imagenet,ba2016layer,ulyanov2016instance,wu2018group,li2019positional}. 
Here, we solely focus on the latter, and we briefly review related work in CV and NLP.

\paragraph{Normalization in Computer Vision}
Batch Normalization (BN)~\cite{ioffe2015batch}
has become the de-facto normalization for NNs used in CV.
BN normalizes the activations (feature maps) by computing channel-wise mean and variance across
the batch dimension, as schematically shown in~\fref{fig:ln-pn-visualition}.
It has been found that BN leads to robustness
with respect to sub-optimal hyperparameters (e.g., learning rate) and initialization, and it generally results in more stable training for CV tasks~\cite{ioffe2015batch}.
Following the seminal work of~\cite{ioffe2015batch}, there have been two principal lines of research:
(i) extensions/modifications of BN to improve its performance,
and (ii) theoretical/empirical studies to understand why BN helps training.

With regard to (i), it was found that BN does not perform well for problems that need
to be trained with small batches, e.g., image segmentation (often due to memory limits)~\cite{zagoruyko2016wide,lin2017focal,goldberger2005neighbourhood}.
The work of~\cite{ioffe2017batch} proposed batch renormalization to remove/reduce
the dependence of batch statistics to batch size. It was shown that this
approach leads to improved performance for small batch training as well as cases
with non-i.i.d. data. 
Along this direction, the work of~\cite{singh2019evalnorm} proposed ``EvalNorm,''
which uses corrected normalization statistics. Furthermore, the recent
work of~\cite{yan2020towards} proposed ``Moving Average Batch Normalization (MABN)'' for small
batch BN by replacing batch statistics with moving averages. 

There has also been work on alternative normalization techniques, and in particular Layer Normalization (LN), proposed by~\cite{ba2016layer}. LN normalizes across the channel/feature dimension as shown in~\fref{fig:ln-pn-visualition}.
This could be extended to Group Norm (GN)~\cite{wu2018group}, where the normalization is performed
across a partition of the features/channels with different pre-defined groups.
Instance Normalization (IN)~\cite{ulyanov2016instance} is another technique, where
per-channel statistics are computed for each sample.

With regard to (ii), there have been several studies to understand why BN helps training in CV.
The original motivation was that BN reduces the so-called ``Internal Covariance Shift'' (ICS)~\cite{ioffe2015batch}. However, this
explanation was viewed as incorrect/incomplete~\cite{alirahimi}. In particular, the
recent study of~\cite{santurkar2018does} argued that the underlying reason that BN helps training
is that it results in a smoother loss landscape. This was later confirmed for deep NN models by measuring
the Hessian spectrum of the network with/without BN~\cite{yao2019pyhessian}.

\paragraph{Normalization in Natural Language Processing } 
Despite the great success of BN in CV, the large computation and storage overhead of BN at each time-step in recurrent neural networks (RNNs) made it impossible/expensive to deploy
for NLP tasks~\cite{cooijmans2016recurrent}. 
To address this, the work of~\cite{cooijmans2016recurrent,hou2019normalization} used
shared BN statistics across different time steps of RNNs. However, it was found that the performance of BN
is significantly lower than LN for NLP. For this reason, LN became the
default normalization technique, even for the recent transformer models introduced by~\cite{vaswani2017attention}.

Only limited recent attempts were made to compare LN with other alternatives or investigate the reasons behind the success of LN in transformer models. 
For instance, \cite{zhang2019root} proposes RMSNorm, which removes the re-centering invariance in LN and performs re-scaling invariance with the root mean square summed of the inputs. 
They showed that this approach achieves similar performance to LN, but with smaller (89\% to 64\%) overhead.
Furthermore,
\cite{nguyen2019transformers} studies different variants of weight normalization for transformers in low-resource machine translation. 
The recent work of \cite{xu2019understanding} studies
why LN helps training, and in particular it finds that the derivatives
of LN help recenter and rescale backward gradients.
From a different angle,~\cite{zhang2019fixup,zhang2019improving} try to ascribe the benefits of LN
to solving the exploding and vanishing gradient problem at the beginning of training. 
They also propose two properly designed initialization schemes which also enjoy that property and are able to stabilize training for transformers. 

However, most of these approaches achieve similar or marginal improvement over LN.
More importantly, there is still not a proper understanding of why BN performs poorly for transformers applied to NLP data.
Here, we address this by systematically studying the BN behavior throughout training; and, based on our results, we propose Power Normalization (\OURS), a new normalization method that significantly outperforms LN for a wide range of tasks in NLP.

%%%%%%%%%%%%%%%%%%%%%%%%%%%%%%%%%%%%%%%%%%%%%%%%%%%%%%%%%%%%%%%%%%%%%%%%
\begin{figure*}[t]
\begin{center}
  {\label{subfig:cvnlp-fwd-bn-mean}\includegraphics[width=.49\linewidth]{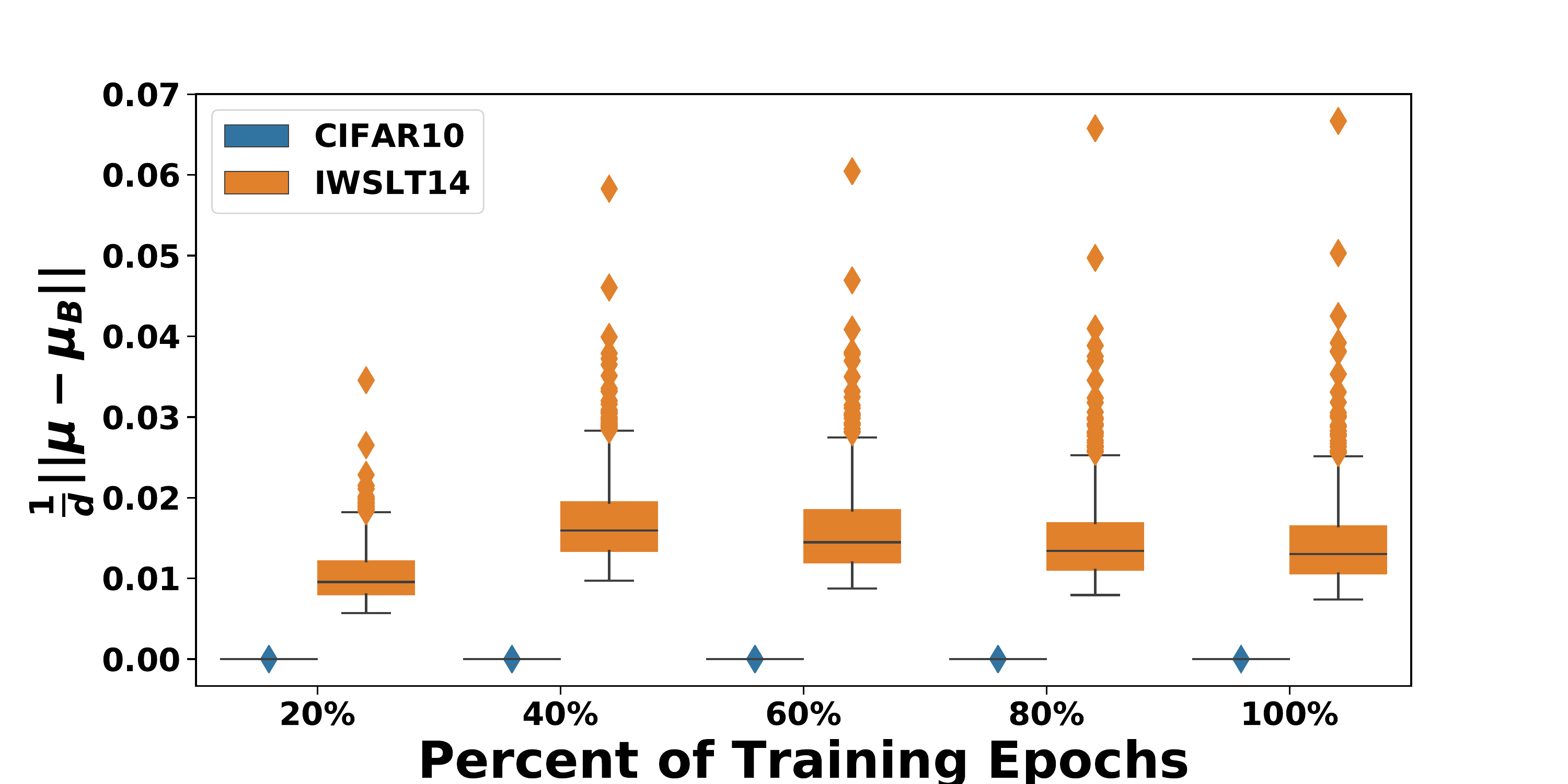}}
  {\label{subfig:cvnlp-fwd-bn-var}\includegraphics[width=.49\linewidth]{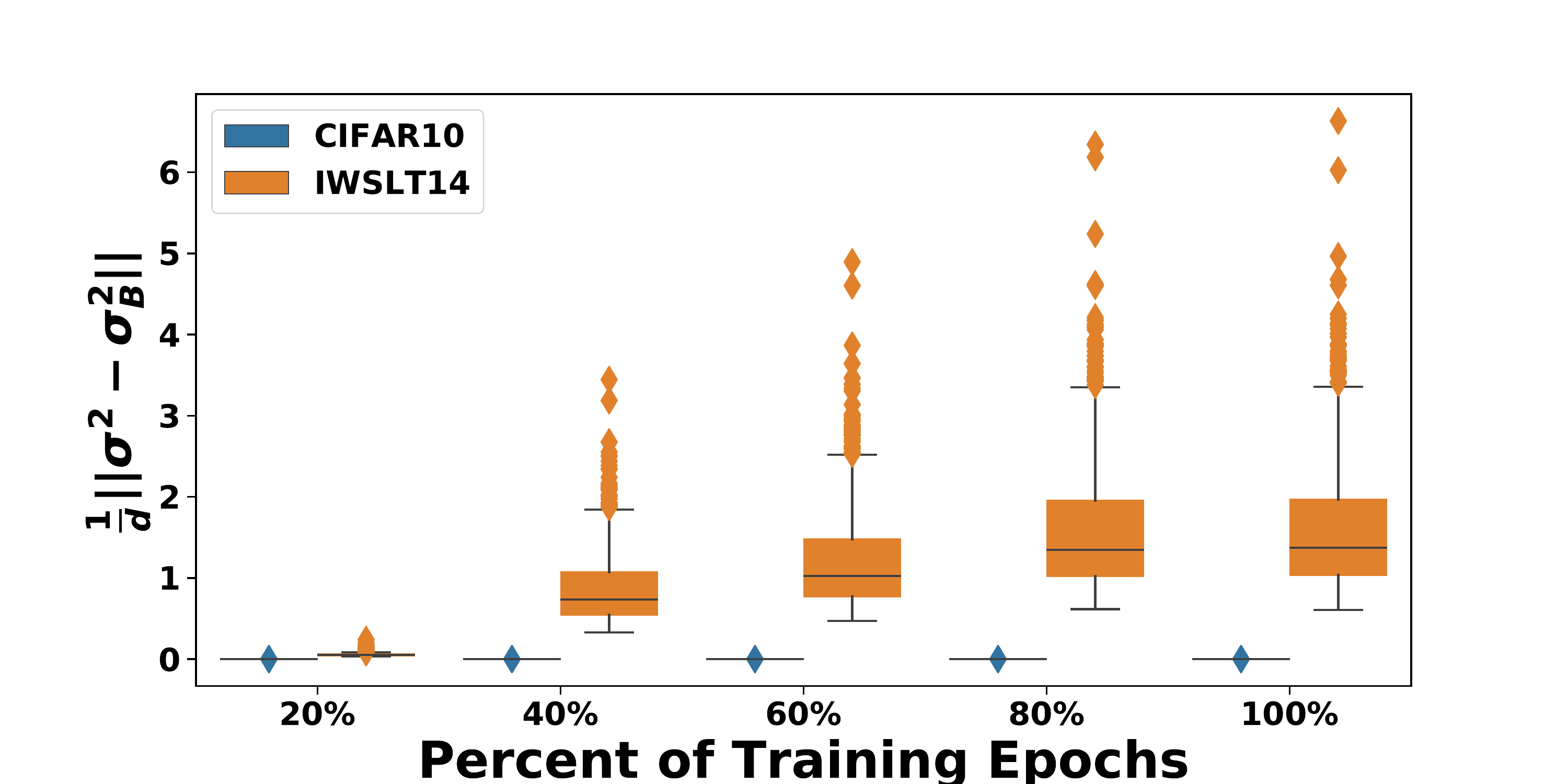}}
\end{center}
 \caption{\footnotesize 
 The average Euclidean distance between the batch statistics ($\mub$, $\sigmab^2$) and the running statistics ($\mu$, $\sigma^2$) stored in first BN during forward pass for ResNet20 on Cifar-10 and Transformer on IWLST14.  
 We can clearly see that the ResNet20 statistics have orders of magnitude smaller variation than the running statistics
 throughout training. However, the corresponding statistics in \transbn exhibit very high
 variance with extreme outliers. This is true both for the mean (shown in the left) as well as variance (shown in right).
 This is one of the contributing factors to the low performance of BN in transformers.
 }
\label{fig:cvnlp-fwd-bn}
\end{figure*}
%%%%%%%%%%%%%%%%%%%%%%%%%%%%%%%%%%%%%%%%%%%%%%%%%%%%%%%%%%%%%%%%%%%%%%%%

\begin{algorithm}[!ht]
\DontPrintSemicolon
\caption{Batch Normalization (Every Iteration)}
\label{alg:batch_normalization}
    \SetAlgoLined
    \Begin(\textbf{Forward Propagation:}){
    
    \KwInput{$\mX\in\mR^{B\times d}$}
    \KwOutput{$\mY\in\mR^{B\times d}$}
    
    $\mub = \frac{1}{B}\sum_{i=1}^B \xi$ \tcp*{Get mini-batch mean}
    
    $\sigmab^2 = \frac{1}{B}\sum_{i=1}^B (\xi-\mub)^2$ \tcp*{Get mini-batch variance}
    
    $\xcheck = \frac{\mX-\mub}{\sigmab}$ \tcp*{Normalize}
    
    $\mY = \gamma\odot\xcheck + \beta$ \tcp*{Scale and shift}
    
    \vspace{0.1in}
    
    $\mu = \alpha\mu + (1-\alpha)\mub$ \tcp*{Update running mean}
    
    $\sigma^2 = \alpha\sigma^2 + (1-\alpha)\sigmab^2$ \tcp*{Update running variance}
    
    }

    \Begin(\textbf{Backward Propagation:}){
    
    \KwInput{$\frac{\partial \loss}{\partial \mY}\in\mR^{B\times d}$}
    \KwOutput{$\frac{\partial \loss}{\partial \mX}\in\mR^{B\times d}$}
    % $\frac{\partial \loss}{\partial \xcheck} = \frac{1}{\gamma}\frac{\partial \loss}{\partial \mY}$ \tcp*{Intermediate Gradient}
    $\frac{\partial \loss}{\partial \mX}$ based on~\eref{eq_m:gradient_loss_x_in_bn} \tcp*{Gradient of $\mX$}
    }
    % \hspace{-.2in}\hrulefill\hspace{.03in}\\

    \hrulefill

    \textbf{Inference: } $\mY=\gamma\odot\frac{\mX-\mu}{\sigma}+\beta$

\end{algorithm}

\section{Batch Normalization}
\label{sec:bkg_batch_normazliation}

\textbf{Notation.}
We denote the input of a normalization layer as $\mX\in \sR^{B\times d}$,
where $d$ is the embedding/feature size and $B$ is the batch size\footnote{For NLP tasks, we flatten sentences/word in one dimension, i.e., the batch size actually corresponds to all \textbf{non-padded} words in a training batch.}. 
We denote $\loss$ as the final loss of the NN. 
The i-th row (column) of a matrix, e.g., $\mX$, is denoted by $\mX_{i, :}$ ($\mX_{:, i}$). 
We %For simplicity, we 
also write the i-th row of the matrix as its lower-case version, i.e., $\xi=\mX_{i, :}$. 
For a vector $\ry$, $y_i$ denotes the i-th element in~$\ry$.

Without other specification: 
(i) for two vector $\rx\in\sR^d$ and $\ry\in\sR^d$, we denote $\rx\ry$ as the element-wise product, $\rx+\ry$ as the element-wise sum, and $\langle \rx, \ry\rangle$ as the inner product; 
(ii) for a vector $\ry\in\sR^d$ and a matrix $\mX\in\sR^{B\times d}$, we denote $\ry\odot\mX$ as $[y_1\mX_{:, 1}, ..., y_d\mA_{:, d}]$ and $\ry+\mX$ as $[\ry+\mX_{1,:};...;\ry+\mX_{B,:} ]$; and
(iii) for a vector $\ry\in\sR^d$, $\ry > C$ means that each entry of $\ry$ is larger than the constant $C$, i.e., $y_i>C$ for all $i$.

\subsection{Formulation of Batch Normalization}
\label{sec:batch_normalization}

We briefly review the formulation of BN~\cite{ioffe2015batch}.
Let us denote the mean (variance) of $\mX$ along the batch dimension as $\mub\in\sR^{d}$ ($\sigmab^2\in\sR^{d}$).
The batch dimension is illustrated in~\fref{fig:ln-pn-visualition}.
The BN layer first enforces zero mean and unit variance, and it then performs an affine transformation by scaling the result by $\gamma,\beta\in\sR^d$, as shown in~\aref{alg:batch_normalization}.

The Forward Pass (FP) of BN is performed as follows. 
Let us denote the intermediate result of BN with zero mean and unit variance as $\xcheck$, i.e.,
\small
\begin{equation}\label{eq:xcheck_defination}
    \xcheck = \frac{\mX-\mub}{\sigmab}.
\end{equation}
\normalsize
The final output of BN, $\mY$, is then an affine transformation applied to $\xcheck$:
\small
\begin{equation}\label{eq:y_defination}
    \mY = \gamma \odot \xcheck + \beta. 
\end{equation}
\normalsize

The corresponding Backward Pass (BP) can then be derived as follows.
Assume that the derivative of $\loss$ with respect to $\mY$ is given, i.e., $\frac{\partial \loss}{\partial \mY}$ is known.
Then, the derivative with respect to input can be computed as:

\small
\begin{equation}
    \label{eq_m:gradient_loss_x_in_bn}
    \frac{\partial \loss}{\partial \xi} = \frac{\gamma}{\sigmab}\frac{\partial \loss}{\partial \vy_i} - \frac{\gamma}{\sigmab B}\sum_{j \in B}(\underbrace{\frac{\partial \loss}{\partial \vy_j}}_\text{from $\mub$: $g_{\mu}$}+
    \underbrace{\frac{\partial \loss}{\partial \vy_j}\xcheckj\xchecki}_\text{from $\sigmab^2$: $g_{\sigma^2}$}).
\end{equation}
\normalsize
See Lemma~\ref{lemma:bn_gradient_x} in Appendix \ref{sec:theoretical_proof} for details.
We denote the terms contributed by $\mub$ and $\sigmab^2$ as $g_{\mu}$ and $g_{\sigma^2}$, respectively.

In summary, there are four batch statistics in BN, two in FP and two in BP. 
The stability of training is highly dependent on these four parameters.
In fact, naively implementing the BN as above for transformers leads to poor performance.
For example, using transformer with BN  (denoted as \transbn) results in 1.1 and 1.4 lower BLEU score, as compared to the transformer with LN (\transln), on IWSLT14 and WMT14, respectively; see~\tref{table:translation}.

This is significant performance degradation, and it stems from instabilities
associated with the above four batch statistics.
To analyze this, we studied the batch statistics using the standard setting of ResNet20  on Cifar-10 and \transbn on IWSLT14 (using a standard batch size of 128 and tokens of 4K, respectively).
In the first experiment, we probed the fluctuations between batch statistics, $\mub$/$\sigmab$, and the corresponding BN running statistics, $\mu$/$\sigma$, throughout training.
This is shown for the first BN layer of ResNet20 on Cifar-10 and \transbn on IWSLT14 in~\fref{fig:cvnlp-fwd-bn}.
Here, the y-axis shows the 
average Euclidean distance between batch statistics ($\mub$, $\sigmab$) and the running statistics ($\mu$, $\sigma$), and the x-axis is different epochs of training, where we define the average Euclidean distance as $\text{dist}(\mub, \mu) = \frac{1}{d}\|\mub-\mu\|$.

The first observation is that \transbn shows significantly larger distances between the batch statistics and the running statistics than ResNet20 on Cifar-10, which exhibits close to zero fluctuations. 
Importantly, this distance between $\sigmab$ and $\sigma$ significantly increases throughout training, but with extreme outliers.
During inference, we have to use the running statistics. However, such large fluctuations would lead
to a large inconsistency between statistics of the testing data and the BN's running~statistics.

The second observation comes from probing the norm of
$g_{\mu}$ and $g_{\sigma^2}$ defined in~\eref{eq_m:gradient_loss_x_in_bn}, which contribute
to the gradient backpropagation of input. 
These results are shown in~\fref{fig:cvnlp-bwd-bn}, where we report
the norm of these two parameters for ResNet20 and \transbn.
For \transbn, we can see very large outliers that actually persist throughout training.
This is in contrast to ResNet20, for which the outliers vanish as training~proceeds.

%%%%%%%%%%%%%%%%%%%%%%%%%%%%%%%%%%%%%%%%%%%%%%%%%%%%%%%%%%%%%%%%%%%%%%%%
\begin{figure*}[t]
\begin{center}
  {\label{subfig:cvnlp-bwd-bn-mean}\includegraphics[width=.49\linewidth]{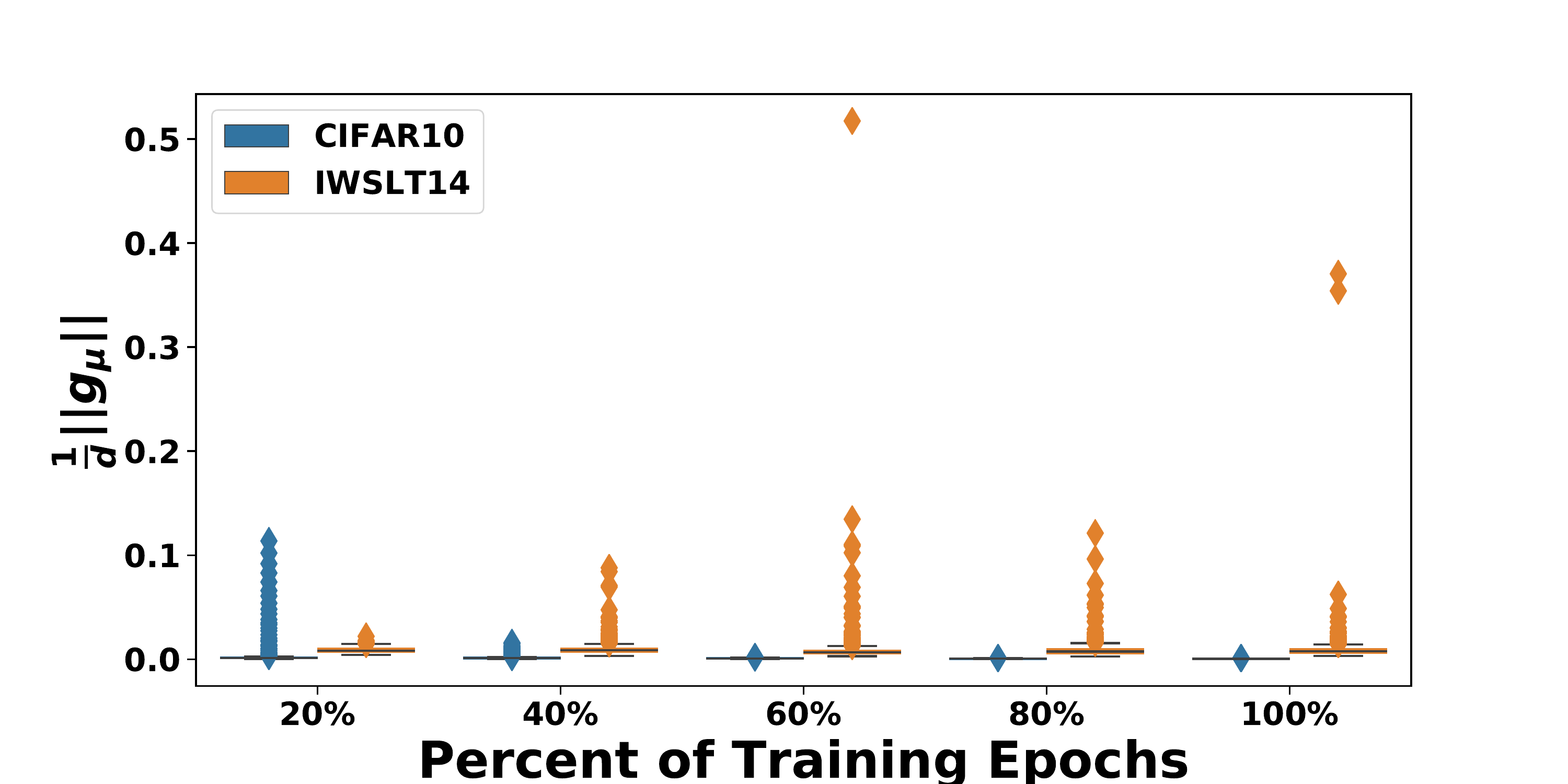}}
  {\label{subfig:cvnlp-bwd-bn-var}\includegraphics[width=.49\linewidth]{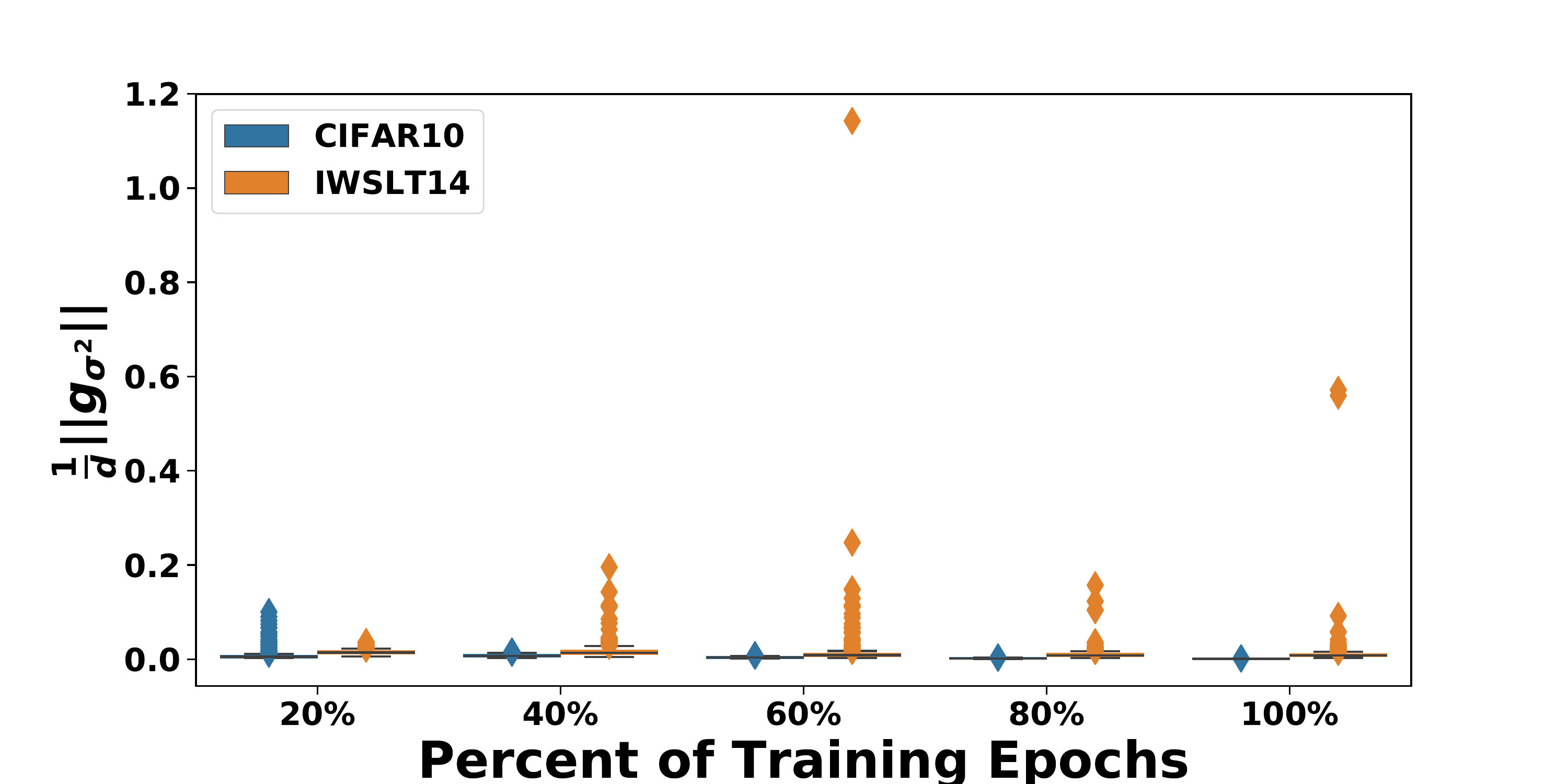}}
\end{center}
 \caption{\footnotesize 
 The average gradient norm of the input of the first BN layer contributed by $\mub$ and $\sigmab$ for ResNet20 on Cifar10 and \transbn on IWSLT14 during the BP (note that $d=16$ for Cifar-10 and $d=512$ for IWSLT experiment). 
 It can be clearly seen that the norm of $g_{\mu}$ and $g_{\sigma^2}$ for ResNet20 has orders of magnitude smaller variation
 throughout training, as compared to that for \transbn. 
 Also, the outliers for ResNet20 vanish at the end of training, which is in contrast to \transbn, for which the outliers persist.  
 This is true both for $g_{\mu}$ (shown in left) as well as $g_{\sigma^2}$ (shown in right).
 }
\label{fig:cvnlp-bwd-bn}
\end{figure*}
%%%%%%%%%%%%%%%%%%%%%%%%%%%%%%%%%%%%%%%%%%%%%%%%%%%%%%%%%%%%%%%%%%%%%%%%
\section{Power Normalization}
\label{sec:method}

\begin{figure*}[t]
\begin{center}
  {\label{subfig:bn-pn-fwd}\includegraphics[width=.49\linewidth]{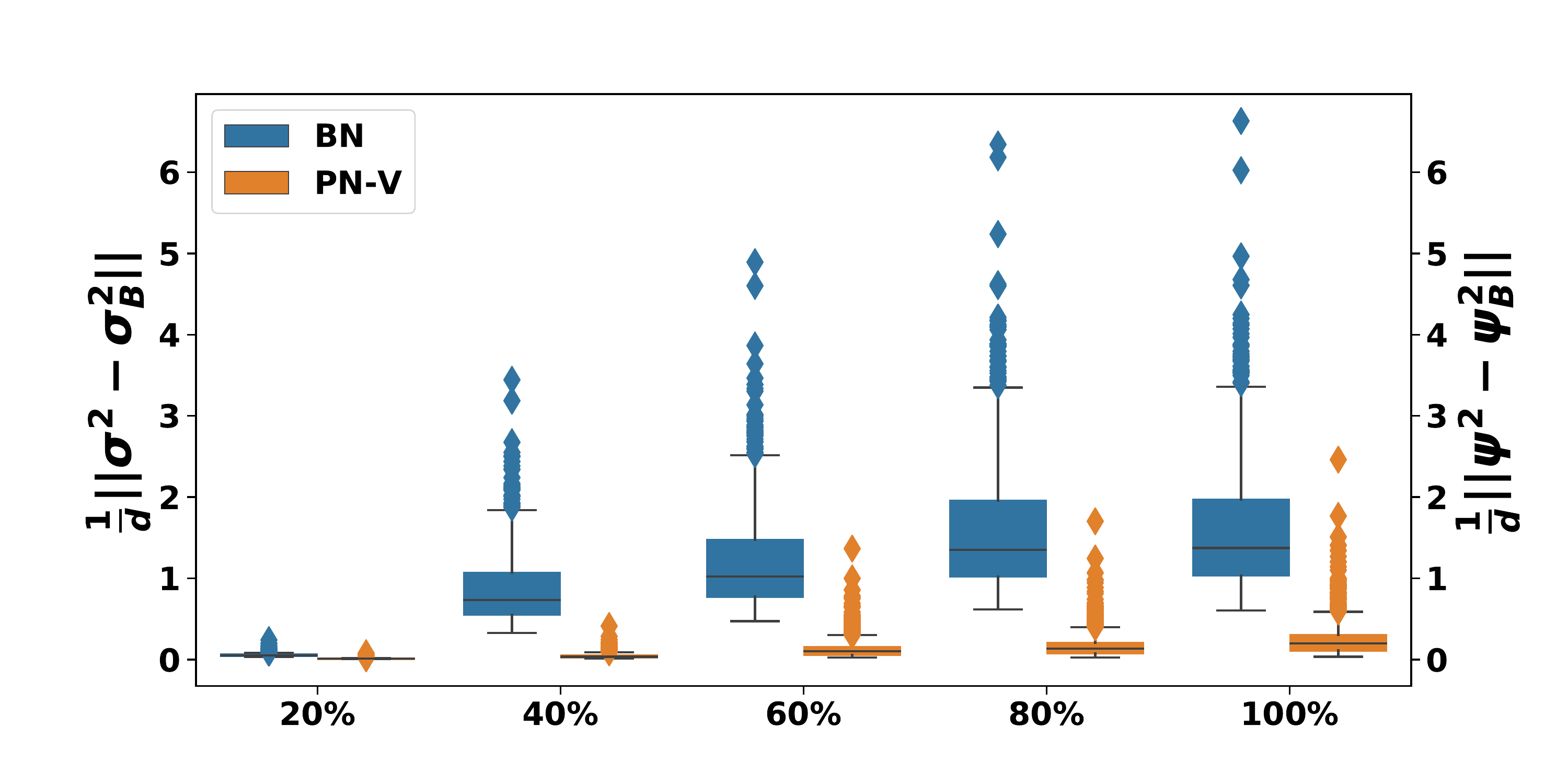}}
  {\label{subfig:bn-pn-bwd}\includegraphics[width=.49\linewidth]{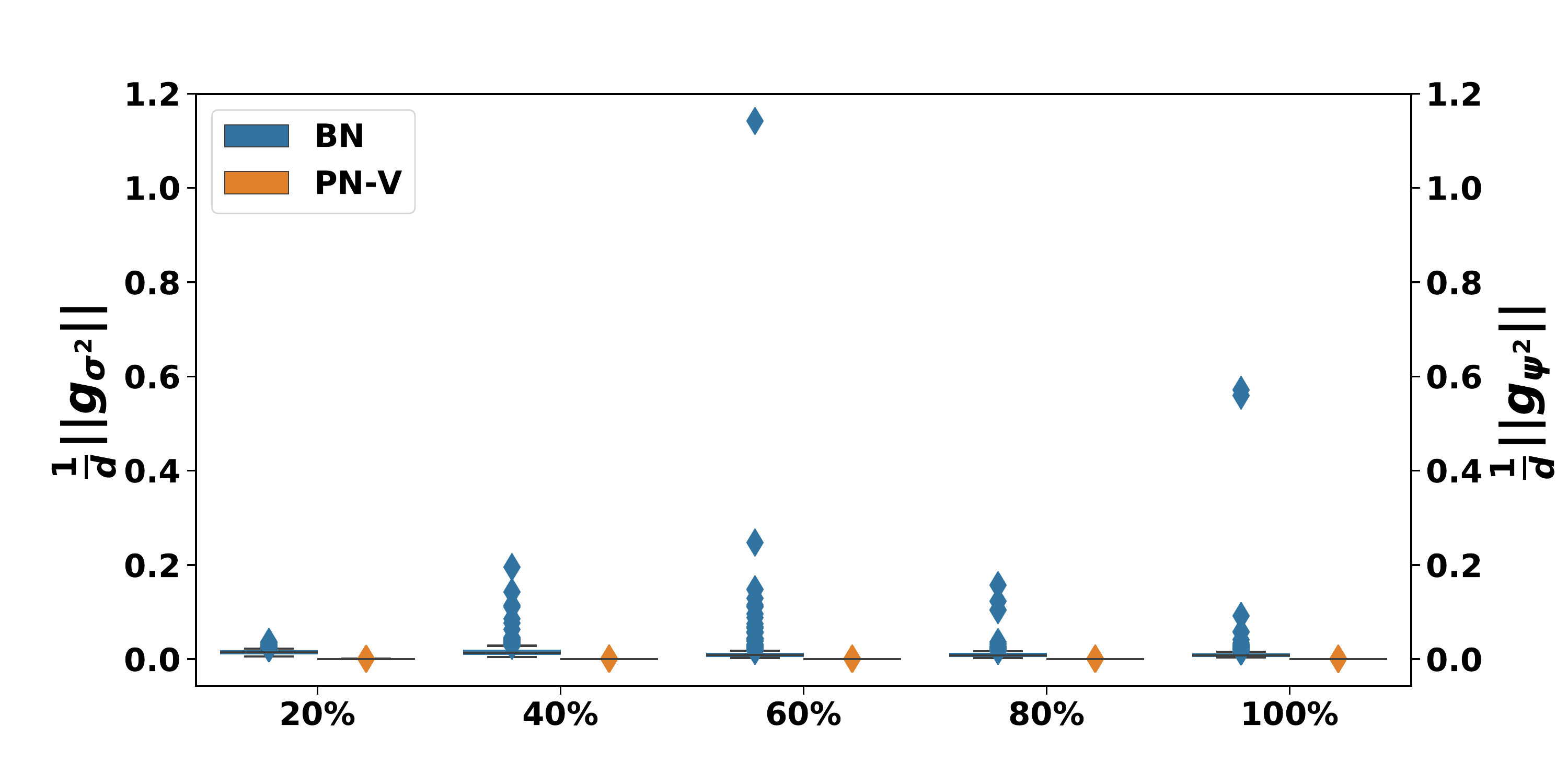}}
\end{center}
 \caption{\footnotesize 
 Results for Transformer on IWSLT14. 
 (Left) The average Euclidean distance between batch statistics ($\sigmab$, $\psib$) and the running statistics ($\sigma$, $\psi$) stored in first BN/\OURSV during forward propagation (FP).   
 (Right) The average norm of gradient of the input of the first BN/\OURSV contributed by $\sigmab$/$\psib$.  
 During FP, $\psib$ has much smaller variations of running statistics, as compared to $\sigmab$, as shown in left,
 It can also be clearly seen that during BP, the norm of $g_{\psi^2}$ exhibits many fewer outliers, as compared to $g_{\sigma^2}$, throughout the training. 
 }
\label{fig:bn-pn-fwd-bwd}
\end{figure*}

Based on our empirical observations, we propose Power Normalization (\OURS), which effectively resolves the performance degradation of BN.
This is achieved by incorporating the following two changes to BN.
First, instead of enforcing unit variance, we enforce unit quadratic
mean for the activations.
The reason for this is that we find that enforcing zero-mean and unit variance in BN is detrimental due to the large variations in the mean, as discussed in the previous section.
However, we observe that unlike mean/variance, the unit quadratic mean is significantly
more stable for transformers. 
Second, we incorporate running statistics for the quadratic mean of the signal,
and we incorporate an approximate backpropagation method to compute the corresponding gradient.
We find that the combination of these two changes leads to a significantly more effective
normalization, with results that exceed LN, even when the same training hyper-parameters are used.
Below we discuss each of these two components.

\subsection{Relaxing Zero-Mean and Enforcing Quadratic Mean}
\label{sec:rm_mean}
Here, we describe the first modification in PN.
As shown in~\fref{fig:cvnlp-fwd-bn} and \ref{fig:cvnlp-bwd-bn}, $\mub$ and $g_\mu$ exhibit significant number of large outliers, which leads to inconsistencies between training and inference statistics. 
We first address this by relaxing the zero-mean normalization, and we use the quadratic mean of the signal, instead of its variance. 
The quadratic mean exhibits orders of magnitude smaller fluctuations, as shown in~\fref{fig:bn-pn-fwd-bwd}.
We refer to this normalization (i.e., no zero mean and unit quadratic mean enforcement) as \OURSV, defined as follows.

\begin{definition}[\OURSV]
Let us denote the quadratic mean of the batch as $\psib^2 = \frac{1}{B}\sum_{i=1}^B \vx_i^2$.
Furthermore, denote $\xhat$ as the signal scaled by $\psib$, i.e.,
\small
\begin{equation}\label{eq:xhat_defination}
    \xhat = \frac{\mX}{\psib}. 
\end{equation}
Then, the output of \OURSV is defined as
\normalsize
\small
\begin{equation}\label{eq:y_defination2}
    \mY = \gamma \odot \xhat + \beta ,
\end{equation}
\normalsize
where $\gamma\in \sR^{d}$ and $\beta\in \sR^{d}$ are two parameters (vectors) in \OURSV (which is the same as in the affine transformation used in BN).

Note that here we use the same notation $\mY$ as the output in~\eref{eq:y_defination} without confusion. 
\end{definition}

The corresponding BP of  \OURSV is as follows: 
\small
\begin{equation}
    \label{eq_m:gradient_loss_x_in_pnv}
    \frac{\partial \loss}{\partial \xi} = \frac{\gamma}{\psib}\frac{\partial \loss}{\partial \vy_i} - \frac{\gamma}{B\psib}\sum_{j\in B}\underbrace{\frac{\partial \loss}{\partial \vy_j} \xhatj\xhati}_\text{from $\psib^2$: $g_{\psi^2}$}  .
\end{equation}
\normalsize
See Lemma~\ref{lemma:oursv_gradient_x} in Appendix \ref{sec:theoretical_proof} for the full details.
Here, $g_{\psi^2}$ is the gradient attributed by $\psib^2$. 
Note that, compared to BN, there exist only two batch statistics in FP and BP: $\psib^2$ and $g_{\psi^2}$.
This modification removes the two unstable factors corresponding to $\mub$ and $\sigmab$ in BN ($g_{\mu}$, and $g_{\sigma^2}$ in~\eref{eq_m:gradient_loss_x_in_bn}).
This modification also results in significant performance improvement, as reported in~\tref{table:translation} for IWSLT14 and WMT14.
By directly replacing BN with \OURSV (denoted as \transpnv), the BLEU score increases from 34.4 to 35.4 on IWSLT14, and 28.1 to 28.5 on WMT14.
These improvements are significant for these two tasks. For example, \cite{zhang2019fixup,zhang2019improving} only improves the BLEU score by 0.1 on IWSLT14. 

As mentioned before, $\psib$ exhibits orders of magnitude smaller variations, as compared to $\sigmab$.
This is shown in~\fref{fig:bn-pn-fwd-bwd},
where we report the distance between the running statistics for $\sigma$, $\text{dist}(\sigmab^2, \sigma^2)$, and $\psi$, $\text{dist}(\psib^2, \psi^2)$.
Similarly during BP, we compute the norm of $g_{\sigma^2}$ and $g_{\psi^2}$, and we report it in~\fref{fig:bn-pn-fwd-bwd} throughout training. 
It can be clearly seen that during BP, the norm of $g_{\psi^2}$ exhibits many fewer outliers as compared to $g_{\sigma^2}$.

In~\cite{santurkar2018does}, the authors provided theoretical results suggesting that employing BN in DNNs can lead to a smaller Lipschitz constant of the loss. 

It can be shown that \OURSV also exhibits similar behaviour, under mild assumptions.
In more detail, let us denote the loss of the NN without normalization as $\losshat$. 
With mild assumptions, \cite{santurkar2018does} shows that the norm of $\frac{\partial \loss}{\partial \mX}$ (with BN) is smaller than the norm of $\frac{\partial \losshat}{\partial \mX}$.
Here, we show that, under the same assumptions, \OURSV can achieve the same results that \bn does. 
See Appendix~\ref{sec:theoretical_proof} for details, including the statement of Assumption~\ref{ass:assumption_for_lipschitz_constant_of_pnv}.

\begin{lemma}[The effect of \OURSV on the Lipschitz constant of the loss]
\label{lemma:lipschitz_constant_of_pnv}
Under Assumption~\ref{ass:assumption_for_lipschitz_constant_of_pnv}, we have
\small
\begin{equation}
    \|\frac{\partial \loss}{\partial \mX_{:, i}}\|^2 = \frac{\gamma^2_i}{(\psib)_i^2}\left(\|\frac{\partial \losshat}{\partial \mX_{:, i}}\|^2 -  \langle\frac{\partial \losshat}{\partial \mX_{:, i}}, \frac{\xhat_{:, i}}{\sqrt B}\rangle^2\right).
\end{equation}
\normalsize
\end{lemma}
See the proof in Appendix~\ref{sec:theoretical_proof}.
Note that $\langle\frac{\partial \losshat}{\partial \mX_{:, i}}, \frac{\xhat_{:, i}}{\sqrt B}\rangle^2$ is non-negative, and hence the Lipschitz constant of $\loss$ is smaller than that of $\losshat$ if $\gamma_i\leq (\psib)_i$.
This is what we observe in practice, as shown in Appendix \ref{sec:extra_results}.

\subsection{Running Statistics in Training}
\label{sec:using_running_stats}

\begin{algorithm}[!ht]
\DontPrintSemicolon
\caption{Power Normalization (Every Iteration)}
\label{alg:power_normalization}
    \SetAlgoLined
    
    \Begin(\textbf{Forward Propagation:}){
    
    \KwInput{$\mX\in\mR^{B\times d}$}
    \KwOutput{$\mY\in\mR^{B\times d}$}
    
    $\psib^2 = \frac{1}{B}\sum_{i=1}^B \xi^2$ \tcp*{Get mini-batch statistics}
    
    $\xhat = \frac{\mX}{\psi}$ \tcp*{Normalize}
    
    $\mY = \gamma\odot\xhat + \beta$ \tcp*{Scale and shift}
    
    \vspace{0.1in}
    
    $\psi^2 = \alpha\psi^2 + (1-\alpha)\psib^2$ \tcp*{Update running statistics}
    }

    \Begin(\textbf{Backward Propagation:}){
    
    \KwInput{$\frac{\partial \loss}{\partial \mY}\in\mR^{B\times d}$}
    \KwOutput{$\frac{\partial \loss}{\partial \mX}\in\mR^{B\times d}$}

    $\frac{\partial \loss}{\partial \xhat} = \gamma\odot\frac{\partial \loss}{\partial \mY}$ \tcp*{Intermediate Gradient}
    
    $\xtilde' = \frac{\partial \loss}{\partial \xhat} - \nu\xhat$ \tcp*{Intermediate Estimated Gradient}
    
    $\frac{\partial \loss}{\partial \mX} = \frac{\xtilde'}{\psi}$ \tcp*{Gradient of $\mX$} 
    
    % $\nu = \nu(1-(1-\alpha)\frac{1}{B}\sum_{i=1}^B \xhati\xhati) + (1-\alpha)\frac{1}{B}\sum_{i=1}^B\frac{\partial \loss}{\partial \xhati}\xhati$
    $\nu = \nu(1-(1-\alpha)\Gamma) + (1-\alpha)\Lambda$ \tcp*{See Definition~\ref{def:power_normalization} for $\Gamma$ and $\Lambda$}
    }
    
    \vspace{0.1in}
    \hrulefill\\
    \vspace{-0.001in}
    \hrulefill

    \textbf{Inference: } $\mY=\gamma\odot\frac{\mX}{\psi}+\beta$
    
\end{algorithm}

Here, we discuss the second modification in \OURS. 
First note that even though \transpnv outperforms \transbn, it still can not match the performance of LN. 
This could be related to the larger number of outliers present in $\psib$, as shown in \fref{fig:bn-pn-fwd-bwd}.  
A straightforward solution to address this is to use running statistics for the quadratic mean (denoted as $\psi^2$), instead of using per batch statistics, since the latter
changes in each iteration. 
However, using running statistics requires modification of the backpropagation, which we described below.

\begin{definition}[\OURS] 
\label{def:power_normalization}
Denote the inputs/statistics at the t-th iteration by $\cdot^{(t)}$, e.g., $\mX^{(t)}$ is the input data at t-th iteration. In the forward propagation, the following equations are used for the calculation:
\small
\begin{align}
    \xhat^{(t)} &= \frac{\mX^{(t)}}{\psi^{(t-1)}}, \label{eq:forward_xhat}\\ 
    \mY^{(t)}   &= \gamma\odot\xhat^{(t)} + \beta, \label{eq:forward_y}\\ 
    (\psi^{(t)})^2  &= (\psi^{(t-1)})^2 + (1-\alpha)(\psib^2 - (\psi^{(t-1)})^2). \label{eq:forward_psi}
\end{align}
\normalsize
Here, $0<\alpha<1$ is the moving average coefficient in the forward propagation, and $\psib$ is the statistic for the current batch. 
Since the forward pass evolves running statistics, the backward propagation cannot be accurately computed---namely, the accurate gradient calculation needs to track back to the first iteration. 
Here, we propose to use the following approximated gradient in backward propagation:
\small
\begin{align}
    (\xtilde^{(t)})' &= \frac{\partial \loss}{\partial \xhat^{(t)}} - \nu^{(t-1)}\odot\xhat^{(t)}, \label{eq:new_grad}\\ 
    \frac{\partial \loss}{\partial \mX^{(t)})} &= \frac{(\xtilde^{(t)})'}{\psi^{(t-1)}}  \label{eq:input_gradient_replace}, \\
    \nu^{(t)} &= \nu^{(t-1)}(1-(1-\alpha)\Gamma^{(t)}) + (1-\alpha)\Lambda^{(t)},
\end{align}
\normalsize
where $\Gamma^{(t)}=\frac{1}{B}\sum_{i=1}^B\xhati^{(t)} \xhati^{(t)}$ and $\Lambda^{(t)} = \frac{1}{B}\sum_{i=1}^B\frac{\partial \loss}{\partial \xhati^{(t)}}\xhati^{(t)}$. 
\end{definition}
This backpropagation essentially uses running statistics by computing the gradient of the loss w.r.t. the quadratic mean of the current batch, rather than using the computationally infeasible method of computing directly the gradient w.r.t. running statistics of the quadratic mean.
Importantly, this formulation leads to bounded gradients which is necessary for convergence, as shown below.

\begin{mytheorem}[Gradient of $\loss$ w.r.t. $\mX$ is bounded in \OURS] 
\label{thm:gradient_xtilde_bound}
For any datum point of $\xtilde$ (i.e. $\xtilde_{i,:}$), the gradients computed from~\eref{eq:new_grad} are bounded by a constant.

Furthermore, the gradient of $\mX_{i,:}$ is also bounded, as given~\eref{eq:input_gradient_replace}.
\end{mytheorem}
See the proof in Appendix \ref{sec:theoretical_proof}. 
The pseudo-code for \OURS algorithm is presented in~\aref{alg:power_normalization}.

\section{Results}
\label{sec:results}

\subsection{Experiment Setup}
We compare our \OURS method with LN and BN for a variety of sequence modeling tasks: Neural Machine Translation (MT); and Language Modeling (LM).
We implement our code for MT using \textit{fairseq-py}~\cite{ott2019fairseq}, and~\cite{ma2019tensorized} for LM tasks.
For a fair comparison, we directly replace the LN in transformers (\transln) with BN (\transbn) or \OURS (\transpn) without varying the position of each normalization layer or changing the training hyperparameters. 

For all the experiments, we use the pre-normalization setting in~\cite{wang2019learning}, where the normalization layer is located right before the multi-head attention module and point-wise feed-forward network module. 
Following~\cite{wang2019learning}, we generally increase the learning rate by a factor of 2.0, relative to the common post-normalization transformer~\cite{vaswani2017attention}. 
Below we discuss tasks specific settings. \footnote{More detailed experimental settings and comparisons between normalization methods are provided in Appendix~\ref{sec:training_details}, \ref{more_comparison_normalization}. }

\paragraph{Neural Machine Translation} 
We evaluate our methods on two widely used public datasets: IWSLT14 German-to-English (De-En) and WMT14 English-to-German (En-De) dataset. 
We follow the settings reported in~\cite{ott2018scaling}. 
We use transformer \texttt{big} architecture for WMT14 (4.5M sentence pairs) and \texttt{small} architecture for IWSLT14 (0.16M sentence pairs). 
For inference, we average the last 10 checkpoints, and we set the length penalty to 0.6/1.0, and the beam size to 4/5 for WMT/IWSLT, following~\cite{ott2019fairseq}. 
All the other hyperparamters (learning rate, dropout, weight decay, warmup steps, etc.) are set identically to the ones reported in the
literature for LN (i.e., we use the same hyperparameters for BN/\OURS).

\paragraph{Language Modeling} 
We experiment on both PTB~\cite{mikolov2011empirical} and Wikitext-103~\cite{merity2016pointer}, which contain 0.93M and 100M tokens, respectively. 
We use three layers tensorized transformer core-1 for PTB and six layers tensorized transformer core-1 for Wikitext-103, following~\cite{ma2019tensorized}. 
Furthermore, we apply the multi-linear attention mechanism with masking, and we report the final testing set perplexity~(PPL).%
\footnote{We also report the validation perplexity in Appendix~\ref{sec:extra_results}.}

\begin{table}
\centering
% \vspace{3pt}
\centerline{
\resizebox{\columnwidth}{!}{\begin{tabular}{lcccc}
\toprule
\multirow{2}{*}{\textbf{Model}} & \textbf{IWSLT14} & \multicolumn{1}{c}{\textbf{WMT14}}      \\ 
 & \texttt{small} & \texttt{big} \\
\midrule
Transformer \cite{vaswani2017attention} & 34.4 & 28.4 \\
DS-Init \cite{zhang2019improving}       & 34.4 & 29.1 \\
Fixup-Init \cite{zhang2019fixup}     & 34.5 & 29.3 \\
Scaling NMT \cite{ott2018scaling}       & /    & 29.3 \\
Dynamic Conv \cite{wu2019pay}           & 35.2 & 29.7 \\
Transformer + LayerDrop \cite{fan2019reducing} & /    & 29.6\\
\midrule
Pre-Norm \transln               & 35.5 & 29.5 \\
Pre-Norm \transbn               & 34.4 & 28.1 \\
Pre-Norm \transpnv              & 35.5 & 28.5 \\
Pre-Norm \transpn               & \textbf{35.9} & \textbf{30.1} \\
\bottomrule
\end{tabular}  } 
}\caption{MT performance (BLEU) on IWSLT14 De-En and WMT14 En-De testsets.
Using \OURSV instead of BN significantly improves the performance, but LN still outperforms. 
However, \OURS achieves much higher BLEU scores, as compared to LN. 
}
\label{table:translation}
\end{table}

\subsection{Experiment Results}

\paragraph{Neural Machine Translation} 
\label{sec:result_mt} 
We use BLEU~\cite{papineni2002bleu} as the evaluation metric for MT.
Following standard practice, we measure tokenized case-sensitive BLEU and case-insensitive BLEU for
WMT14 En-De and IWSLT14 De-En, respectively. 
For a fair comparison, we do not include other external datasets.
All the transformers in~\tref{table:translation} are using six encoder layers and six decoder layers.

The results  are reported in~\tref{table:translation}. 
In the first
section of rows, we report state-of-the-art results for these two tasks with comparable model sizes.
In the second section of rows, we report the results with different types of normalization.
Notice the significant drop in BLEU score when BN is used (34.4/28.1), as opposed to LN (35.5/29.5).
Using \OURSV instead of BN helps reduce this gap, but LN still outperforms.
However, the results corresponding to \OURS exceeds LN results by more than 0.4/0.6 points,
This is significant for these tasks.
Comparing with other concurrent works like DS-Init and Fixup-Init~\cite{zhang2019improving,zhang2019fixup}, the improvements in \transpn are still significant. 

\paragraph{Language Modeling}
\label{sec:result_lm}

\begin{table}[pthb]\small
  \renewcommand\arraystretch{1.2}
  \centering
  \resizebox{\columnwidth}{!}{\begin{tabular}{lcccc}
    \toprule[1pt]
    \multirow{2}{*}{\textbf{Model}}& \multicolumn{1}{c}{\textbf{PTB}} & \multicolumn{1}{c}{\textbf{WikiText-103}}\\
    \cline{2-3}
    &Test PPL&Test PPL\\
    \hline
    Tied-LSTM~\cite{inan2016tying} & 48.7 & 48.7 \\
    AWD-LSTM-MoS~\cite{yang2017breaking} & 56.0 & 29.2 \\
    \hline
    Adaptive Input~\cite{baevski2018adaptive} & 57.0  & 20.5\\
    Transformer-XL$_{\texttt{base}}$ \cite{dai2019transformer} & 54.5 & 24.0\\
    Transformer-XL$_{\texttt{large}}$ \cite{dai2019transformer} & -- & 18.3 \\
    Tensor-Transformer$_{\text{1core}}$ \cite{ma2019tensorized} & 57.9 & 20.9\\
    Tensor-Transformer$_{\text{2core}}$ \cite{ma2019tensorized} & 49.8 & 18.9\\
    \hline
    Tensor-Transformer$_{\text{1core}}$ + LN & 53.2* & 20.9* \\
    Tensor-Transformer$_{\text{1core}}$ + BN & 60.7  & 27.2 \\
    Tensor-Transformer$_{\text{1core}}$ +~\OURSV & 55.3 & 21.3 \\
    Tensor-Transformer$_{\text{1core}}$ +~\OURS & \textbf{47.6} & \textbf{17.9} \\
    \bottomrule[1pt]
  \end{tabular}}
  \caption{
  Results with state-of-the-art methods on PTB and WikiText-103. '-' indicates no reported results in that setting, '$*$' indicates the results are from our own implementation. 
  \OURS achieves 5.6/3 points lower testing PPL on PTB and WikiTest-103, respectively, compared to LN. 
  } 
  \label{table:language_model}
\end{table}

We report the LM results in~\tref{table:language_model}, using the tensorized transformer proposed in~\cite{ma2019tensorized}. 
Here we observe a similar trend. Using BN results in a significant degradation, increasing testing PPL by more than 7.5/6.3 for PTB/WikiText-103 datasets
(achieving 60.7/27.2 as opposed to 53.2/20.9).
However, when we incorporate the \OURS normalization, we achieve state-of-the-art results for these two tasks (for these model sizes and without
any pre-training on other datasets). In particular, \OURS results in 5.6/3 points lower testing PPL, as compared to LN.
Importantly, note that using \OURS we achieve better results than \cite{ma2019tensorized}, with the same number of parameters.

\subsection{Analysis}
\label{sec:analysis}
\paragraph{The Effect of Batch Size for Different Normalization}
\begin{figure}[!htp]
\begin{center}
  \includegraphics[width=.9\linewidth]{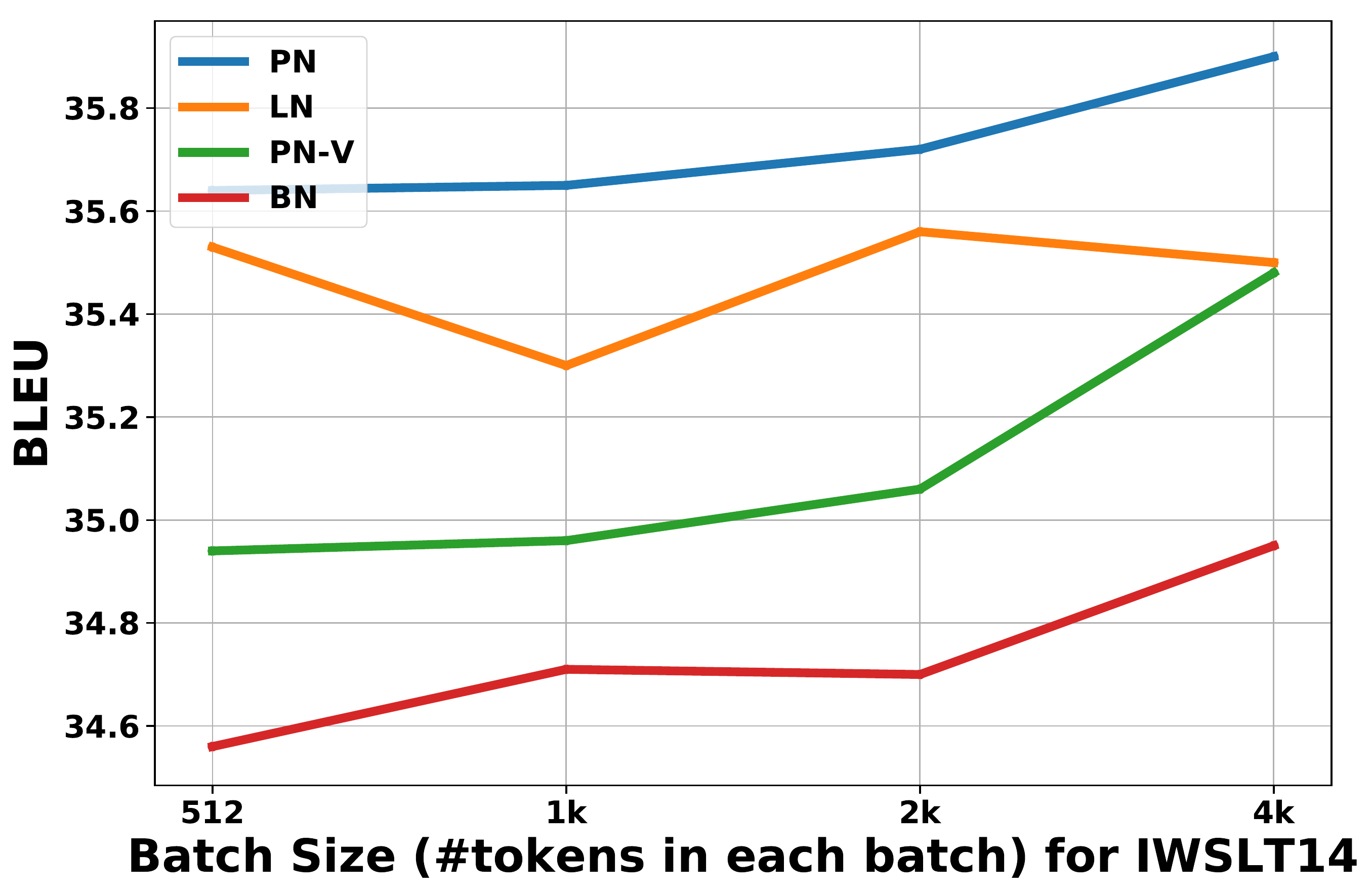}
\end{center}
 \caption{\footnotesize Ablation study of the performance of \OURS, \OURS-V, LN and BN on IWSLT14 trained using different batch sizes. 
 Note that the performance of \OURS consistently outperforms LN. 
 In the meanwhile, \OURSV can only match the result of LN when mini-batch gets to 4K. 
 Among all the settings, BN behaves poorly and abnormally across different mini-batches.
 }
\label{figure:ablation}
\end{figure}

To understand better the effects of our proposed methods \OURS and \OURSV,
we change the batch size used to collect statistics in BN, LN, and \OURS.
To this end, we keep the total batch size constant at 4K tokens,
and we vary the mini-batch size used to collect statistics from 512 to 4K.
Importantly, note that we keep the total batch size constant at 4K, and we use gradient
accumulation for smaller mini-batches. For example, for the case with mini-batch of 512, we use
eight gradient accumulations.
The results are reported in~\fref{figure:ablation}.
We can observe that BN behaves poorly and abnormally across different mini-batches.
Noticeably, after relaxing the zero-mean normalization in BN and replacing the variance estimation with quadratic mean,
\OURSV matches the performance of LN for 4K mini-batch and consistently outperforms BN. However, it underperforms LN.
In contrast, we can see that \OURS consistently achieves higher results under different mini-batch settings.

\paragraph{Representation Power of learned Embedding}
To investigate further the performance gain of \OURS, we compute the Singular Value Decomposition of the embedding layers, as proposed by~\cite{gao2019representation}, which argued that the singular value distribution could be used as a proxy for measuring representational power of the embedding layer.
It has been argued that having fast decaying singular values leads to limiting the representational power of the embeddings to a small sub-space. 
If this is the case, then it may be preferable to have a more uniform singular value distribution~\cite{wang2020improving}.
We compute the singular values for word embedding matrix of LN and \OURS, and we report the results in~\fref{figure:ln-pn-emb-svd}. 
It can be clearly observed that the singular values corresponding to \OURS decay more slowly than those of LN. 
Intuitively, one explanation for this might be that \OURS helps by normalizing all the tokens across the batch dimension, which can result in a more equally distributed embeddings.
This may illustrate one of the reasons why \OURS outperforms LN.

\begin{figure}[!htp]
\begin{center}
  \includegraphics[width=.9\linewidth]{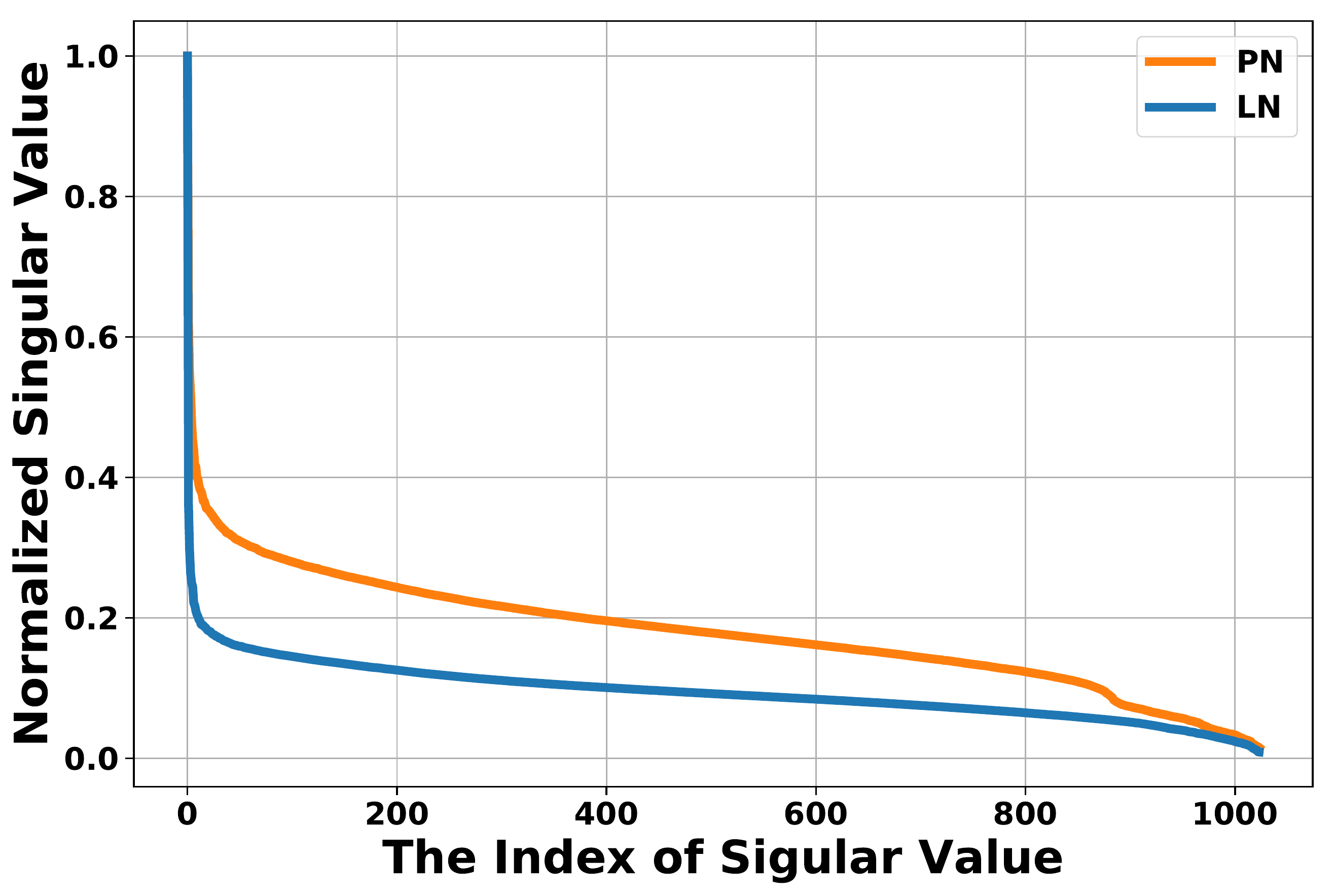}
\end{center}
 \caption{\footnotesize Singular values of embedding matrix trained with LN/\OURS~on WMT14. We normalize the singular values of each matrix so that they are comparable with the largest one as 1. 
 Note that the singular values corresponding to \OURS decay more slowly than those of LN.
 }
\label{figure:ln-pn-emb-svd}
\end{figure}
\section{Conclusion}
\label{sec:conclusion}

In this work, we systematically analyze the ineffectiveness of vanilla batch normalization (BN) in transformers.
Comparing NLP and CV, we show evidence that the batch statistics in transformers on NLP tasks have larger variations. 
This further leads to the poor performance of BN in transformers. 
By decoupling the variations into FP and BP computation, we propose \OURSV and \OURS to alleviate the variance issue of BN in NLP. 
We also show the advantages of \OURSV and \OURS, both theoretically and empirically. 
Theoretically, \OURSV preserves the first-order smoothness property as in BN.
The approximate backpropagation of \OURS leads to bounded gradients. 
Empirically, we show that \OURS outperforms LN in neural machine translation (0.4/0.6 BLEU on IWSLT14/WMT14) and language modeling (5.6/3.0 PPL on PTB/WikiText-103) by a large margin.
We also conduct further analysis of the effect of \OURSV/\OURS/BN/LN under different batch size settings to show the significance of statistical estimations, and we investigate the representation power of learned embeddings matrix by LN/\OURS to illustrate the effectiveness of~\OURS.

\section*{Acknowledgments}
This work was supported by funds from Intel and Samsung.
We are grateful to support from Google Cloud, Google TFTC team, as well as  support from the Amazon AWS.
We would like to acknowledge ARO, DARPA, NSF, and ONR for providing partial support of this~work.
We are also grateful to Zhuohan Li, Zhen Dong, Yang Liu, the members of Berkeley NLP, and the members of the Berkeley RISE Lab for their valuable feedback.
% \clearpage
\bibliography{journal-abbrv,ref}
\bibliographystyle{icml2020}

% \clearpages
 \onecolumn
\appendix
\section{Training Details}
\label{sec:training_details}
\subsection{Machine Translation.}
\paragraph{Dataset} 
The training/validation/test sets for the IWSLT14 dataset contain about 153K/7K/7K sentence pairs, respectively. 
We use a vocabulary of 10K tokens based on a joint source and target byte pair encoding (BPE)~\cite{sennrich2016neural}. 
For the WMT14 dataset, we follow the setup of~\cite{vaswani2017attention}, 
which contains 4.5M training parallel sentence pairs. 
Newstest2014 is used as the test set, and Newstest2013 is used as the validation set. 
The 37K vocabulary for WMT14 is based on a joint source and target BPE factorization. 
\paragraph{Hyperparameter} 
Given the unstable gradient issues of decoders in NMT~\cite{zhang2019improving}, we only change all the normalization layers in the 6 encoder layers from LN to BN/\OURS, and we keep all the 6 decoder layers to use LN. 
For \transpnv\bg and \transbn\bg (not \transpn\bg), we use the synchronized version, where each FP and BP will synchronize the mean/variance/quadratic mean of different batches at different nodes.  
For \OURS, we set the $\alpha$ in the forward and backward steps differently, and we tune the best setting over 0.9/0.95/0.99 on the validation set. 
To control the scale of the activation, we also involve a layer-scale layer~\cite{zhang2019root} in each model setting before the normalization layer. 
The warmup scheme for accumulating $\psi$ is also employed, as suggested in~\cite{yan2020towards} . 
Specifically, we do not tune the warmup steps, but we set it identical to the warmup steps for the learning rate schedule in the optimizer~\cite{vaswani2017attention}. 
We set dropout as 0.3/0.0 for Transformer \bg/\sm model, respectively. 
We use the Adam optimizer and follow the optimizer setting and learning rate schedule in~\cite{wang2019learning}. 
We set the maximum number of updates following~\cite{ott2018scaling} to be 300k for WMT and 100k for IWSLT. We used early stopping to stop the experiments by showing no improvement over the last 10/5 epochs. 
For the \bg model, we enlarge the batch size and learning rate, as suggested in~\cite{ott2019fairseq}, to accelerate training. 
We employ label smoothing of value $\epsilon_\text{ls} = 0.1$ in all experiments.
We implement our code for MT using \textit{fairseq-py}~\cite{ott2019fairseq}. 

\paragraph{Evaluation} 
We use BLEU\footnote{https://github.com/moses-smt/mosesdecoder/blob/master/scripts/generic/multi-bleu.perl}~\cite{papineni2002bleu} as the evaluation metric for MT.
Following standard practice, we measure tokenized case-sensitive BLEU and case-insensitive BLEU for WMT14 En-De and IWSLT14 De-En, respectively. 
For a fair comparison, we do not include other external datasets. 
For inference, we average the last 10 checkpoints, and we set the length penalty to 0.6/1.0 and beam size to 4/5 for WMT/IWSLT, following~\cite{ott2019fairseq}. 

\subsection{Language Modeling.}
\paragraph{Dataset} 
PTB~\cite{mikolov2011empirical} has 0.93M training tokens, 0.073M validation words, and 0.082M test word. 
Wikitext-103~\cite{merity2016pointer} contains 0.27M unique tokens, and 100M training tokens from 28K articles, with an average length of 3.6K tokens per article. 
We use the same evaluation scheme that was provided in~\cite{dai2019transformer}. 
\paragraph{Hyperparameter} 
We use three layers tensorized transformer core-1 for PTB and six layers tensorized transformer core-1 for Wikitext-103, following~\cite{ma2019tensorized}.
This means there exists only one linear projection in multi-linear attention. 
We replace every LN layer with a \OURS layer. 
For \OURS, we set the $\alpha$ in forward and backward differently, and we tune the best setting over 0.9/0.95/0.99 on the validation set. 
The warmup scheme and layer-scale are also the same as the hyperparameter setting introduced for machine translation. 
We set the dropout as 0.3 in all the datasets. 
The model is trained using 30 epochs for both PTB and WikiText-103. 
We use the Adam optimizer, and we follow the learning rate setting in~\cite{ma2019tensorized}. 
We set the warmup steps to be 4000 and label smoothing to be $\epsilon_\text{ls} = 0.1$ in all experiments.

\section{Extra Results}
\label{sec:extra_results}
\subsection{Empirical Results for Lemma~\ref{lemma:lipschitz_constant_of_pnv}.}
Under Assumption~\ref{ass:assumption_for_lipschitz_constant_of_pnv}, mentioned in~\sref{sec:rm_mean} and discussed in Appendix~\ref{sec:proof_pnv_all}, we show 
\begin{equation}
    \|\frac{\partial \loss}{\partial \mX_{:, i}}\|^2 = \frac{\gamma^2_i}{(\psib)_i^2}\big(\|\frac{\partial \losshat}{\partial \mX_{:, i}}\|^2 -  \langle\frac{\partial \losshat}{\partial \mX_{:, i}}, \frac{\xhat_{:, i}}{\sqrt B}\rangle^2\big). \nonumber
\end{equation}
\normalsize

Given that $\langle\frac{\partial \losshat}{\partial \mX_{:, i}}, \frac{\xhat_{:, i}}{\sqrt B}\rangle^2$ is non-negative, the Lipschitz constant of $\loss$ is smaller than that of $\losshat$ if $\gamma_i\leq (\psib)_i$. 
Here, we report the empirical results to show that $\gamma_i\leq (\psib)_i$ holds for each $i\in\{1, 2, ..., d\}$ on IWSLT14; see~\fref{figure:gama-empirical}. 
Observe that the Lipschitz constant of $\loss$ is smaller than that of $\losshat$ empirically in our setting.

\begin{figure}[!htp]
\begin{center}
  \includegraphics[width=.9\linewidth]{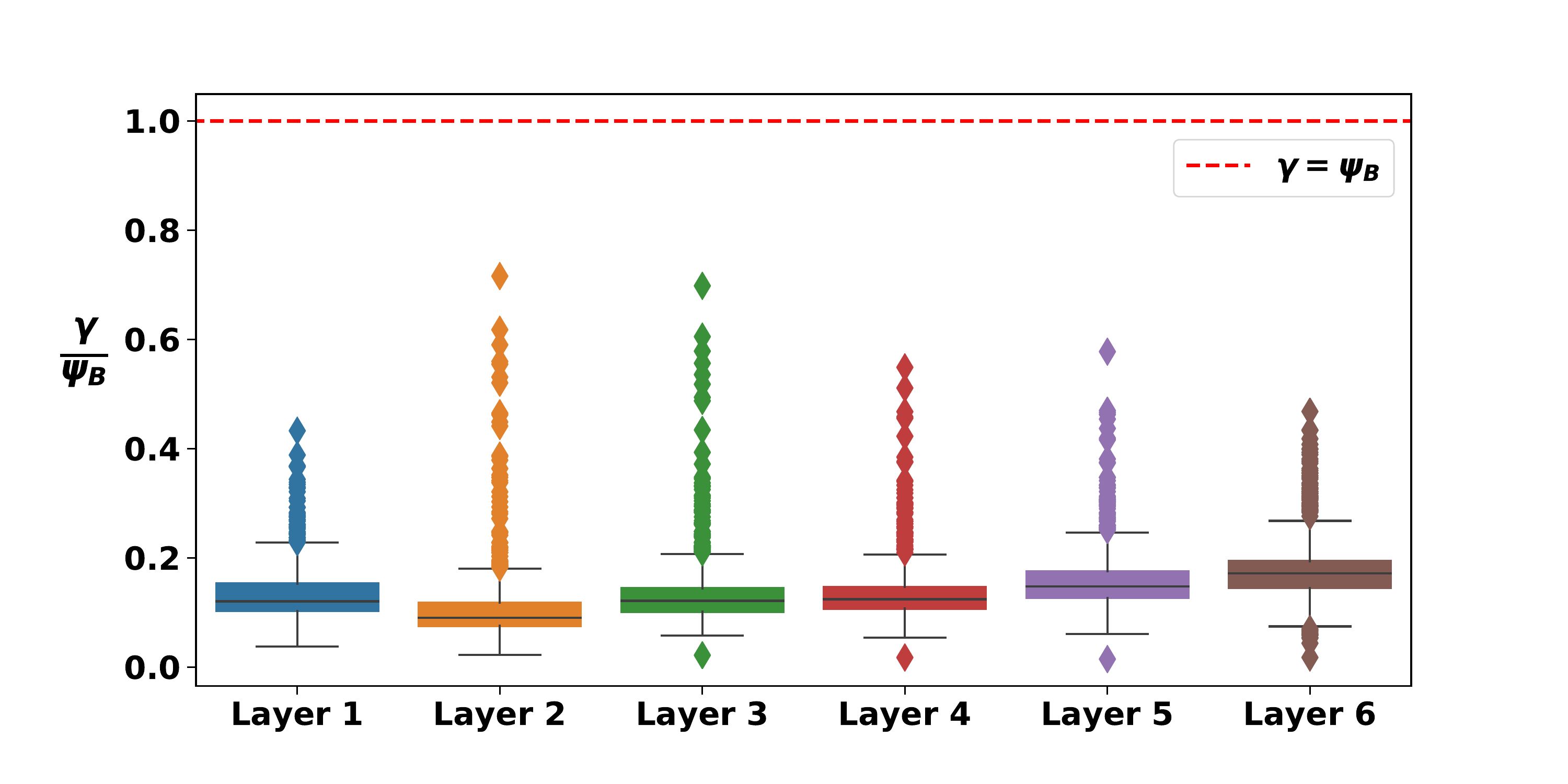}
\end{center}
 \caption{\footnotesize The empirical results of the distribution of $\frac{\gamma}{(\psib)} \in \sR^d$ in different layers of \transpnv on IWSLT14. Given that  $\gamma_i\leq (\psib)_i$ holds for each $i\in\{1, 2, ..., d\}$, \textbf{Lemma \ref{lemma:lipschitz_constant_of_pnv}} holds as well. 
 }
\label{figure:gama-empirical}
\end{figure}
\subsection{Validation Results on Language Modeling.}
\begin{table}[pthb]\small
  \centering
  {\begin{tabular}{lcccccc}
    \toprule[1pt]
    \multirow{2}{*}{\textbf{Model}}& \multicolumn{2}{c}{\textbf{PTB}} & \multicolumn{2}{c}{\textbf{WikiText-103}}\\
    \cline{2-5} 
    &Val PPL&Test PPL&Val PPL&Test PPL\\
    \hline
    Tied-LSTM~\cite{inan2016tying} & 75.7 & 48.7 & -- & 48.7 \\
    AWD-LSTM-MoS~\cite{yang2017breaking} & 58.1 & 56.0 & 29.0 & 29.2 \\
    \hline
    Adaptive Input~\cite{baevski2018adaptive} & 59.1 & 57.0  & 19.8 & 20.5\\
    Transformer-XL$_{\texttt{base}}$ \cite{dai2019transformer} & 56.7 & 54.5 & 23.1 & 24.0\\
    Transformer-XL$_{\texttt{large}}$ \cite{dai2019transformer} & -- & -- & -- & 18.3 \\
    Tensor-Transformer$_{\text{1core}}$ \cite{ma2019tensorized} & 55.4 & 57.9 & 23.6 & 20.9\\
    Tensor-Transformer$_{\text{2core}}$ \cite{ma2019tensorized} & 54.3 & 49.8 & 19.7 & 18.9\\
    \hline
    Tensor-Transformer$_{\text{1core}}$ + LN & 58.0* & 53.2* & 22.7* & 20.9* \\
    Tensor-Transformer$_{\text{1core}}$ + BN & 71.7 & 60.7 & 28.4 & 27.2 \\
    Tensor-Transformer$_{\text{1core}}$ +~\OURSV & 59.7 & 55.3 & 23.6 & 21.3 \\
    Tensor-Transformer$_{\text{1core}}$ +~\OURS & \textbf{51.6} & \textbf{47.6} & \textbf{18.3} & \textbf{17.9} \\
    \bottomrule[1pt]
  \end{tabular}}
  \caption{
  Additional Validation and Test results with state-of-the-art results on PTB and WikiText-103. '-' indicates no reported results in that setting, '$*$' indicates that the results are from our own implementation. 
  \OURS achieves 5.6/3.0 points lower testing PPL on PTB and WikiTest-103, respectively, as compared to LN. 
  } 
  % \vspace{-11px}
  \label{table:language_model_app}
\end{table}

\subsection{More Comparisons.}\label{more_comparison_normalization}
As shown in \ref{normalization_more_compare}, we present more comparison results for different normalization method including Moving Average Batch Normalizatio (MABN) \citep{yan2020towards}, Batch Renormalization (BRN) \cite{ioffe2017batch}, and Group Normalization (GN) \cite{wu2018group}. 
We can observe that BRN/MABN/PN-V is better than BN but worse than LN, which suggests the small batch-size setting (main focus of \cite{yan2020towards,ioffe2017batch,wu2018group}) may have similar characteristic of the setting in NLP, where there exists large variance across batches. 
Obviously, GN performs the best among the previous proposed methods given LN can be viewed as the special case of GN (group number as 1).  \footnote{We empirically found that setting group number the same as head number leads to the best performance.} 
Throughout the comparisons, PN still performs the best in the two tasks, which may validate the effectiveness of our method. 
\begin{table}\label{normalization_more_compare}
    \centering
    % \footnotesize
    \begin{tabular}{lc|c}
    \toprule
    {\textbf{Model}} & \textbf{IWSLT14} & \textbf{PTB}  \\ 
    % & \texttt{small}  & \texttt{3-Layer} \\
    \midrule
    \transbn                                       & 34.4  & 60.7 \\
    Transformer$_{\textsc{BRN}}$ & 34.7  & 58.3 \\
    Transformer$_{\textsc{MABN}}$ & 34.9  & 57.2 \\
    \transln       & 35.5  & 53.2               \\
    Transformer$_{\textsc{GN}}$ & 35.7  & 51.7 \\
    \midrule
    \transpnv                                    & 35.5    & 55.3  \\
    \transpn                                    & \bf{35.9}    & \bf{47.6}  \\
    \bottomrule
    \end{tabular}
    \caption{
    (Left) NMT performance (BLEU) on IWSLT14 De-En. 
    (Right) LM performance (Test PPL) on PTB.
    }
\end{table}

\section{Theoretical Results}
\label{sec:theoretical_proof}
In this section, we discuss the theoretical results on BN and \OURS. 
We assume $\gamma$ and $\beta$ to be constants for our analysis on BN, \OURSV and \OURS. 

Since the derivative of loss $\loss$ w.r.t. $\mY$ is known as $ \frac{\partial \loss}{\partial \mY}$, 
trivially, we will have $\gamma \odot\frac{\partial \loss}{\partial \xcheck} = \frac{\partial \loss}{\partial \mY}$. 
Also, it is not hard to get the following fact.
\begin{fact} \label{fact:derivative_mu_sigma}
The derivatives of $\mub$ and $\sigmab^2$ w.r.t. $\xi$ are
\begin{equation}
    \frac{\partial \mub}{\partial \xi} = \frac{1}{B}~~~\text{and}~~~\frac{\partial \sigma^2}{\partial \xi} = \frac{2}{B}(\xi - \mub).
\end{equation}

\end{fact}
We are now ready to show the derivative of $\loss$ w.r.t. $\xi$ under \bn. 

\begin{lemma}[Derivative of $\loss$ w.r.t. $\xi$ in \BN]
\label{lemma:bn_gradient_x}
Based on the Fact~\ref{fact:derivative_mu_sigma}, it holds that
\begin{equation}
    \frac{\partial \loss}{\partial \xi} = \frac{1}{\sigmab}\frac{\partial \loss}{\partial \xchecki} - \frac{1}{\sigmab B}\sum_{j \in B}\frac{\partial \loss}{\partial \xcheckj}(1+\xcheckj\xchecki).
\end{equation}
\end{lemma}
\begin{proof} Based on chain rule, we will have
\begin{equation}
    \begin{split}
        \frac{\partial \loss}{\partial \xi} 
        &= \frac{\partial \loss}{\partial \xchecki} \frac{\partial\xchecki}{\partial\xi} + \sum_{j\in B}(\frac{\partial \loss}{\partial \xcheckj} \frac{\partial\xcheckj}{\partial\mub}\frac{\partial\mub}{\partial\xi} + \frac{\partial \loss}{\partial \xcheckj} \frac{\partial\xcheckj}{\partial\sigmab}\frac{\partial\sigmab}{\partial\xi}) \\
        &= \frac{1}{\sigmab}\frac{\partial \loss}{\partial \xchecki} + \sum_{j\in B}\frac{\partial \loss}{\partial \xcheckj}( \frac{\partial\xcheckj}{\partial\mub}\frac{1}{B} + \frac{\partial\xcheckj}{\partial\sigmab^2}\frac{2}{B}(\xi - \mub))\\
        &= \frac{1}{\sigmab}\frac{\partial \loss}{\partial \xchecki} - \frac{1}{\sigmab B}\sum_{j \in B}\frac{\partial \loss}{\partial \xcheckj}(1+\frac{\xi-\mub}{\sigmab}\frac{\xj-\mub}{\sigmab})\\
        &= \frac{1}{\sigmab}\frac{\partial \loss}{\partial \xchecki} - \frac{1}{\sigmab B}\sum_{j \in B}\frac{\partial \loss}{\partial \xcheckj}(1+\xcheckj\xchecki).
    \end{split}
\end{equation}
\end{proof}

Replacing $\frac{\partial \loss}{\partial \xcheck}$ by $\gamma \odot \frac{\partial \loss}{\partial \mY}$, we can get~\eref{eq_m:gradient_loss_x_in_bn}.

In the following, we will first discuss the theoretical properties of \OURSV in Appendix~\ref{sec:proof_pnv_all}; and then we discuss how to use running statistics in the forward propagation and how to modify the corresponding backward propagation in Appendix~\ref{sec:proof_pn_all}.

\subsection{Proof of \OURSV}
\label{sec:proof_pnv_all}

Before showing the gradient of $\loss$ w.r.t. $\xi$ under \OURSV, we note the following fact, which is not hard to establish.
\begin{fact} \label{fact:derivative_psi}
The derivatives of $\psib$ w.r.t. $\xi$ are,
\begin{equation}
    \frac{\partial \psib^2}{\partial \xi} = \frac{2}{B}\xi.
\end{equation}
\end{fact}

With the help of Fact~\ref{fact:derivative_psi}, we can prove the following lemma
\begin{lemma}[Derivative of $\loss$ w.r.t. $\xi$ in \OURSV]
\label{lemma:oursv_gradient_x}
Based on the Fact~\ref{fact:derivative_psi}, it holds that that
\begin{equation}
    \frac{\partial \loss}{\partial \xi} = \frac{1}{\psib}\frac{\partial \loss}{\partial \xhati} - \frac{1}{B\psib}\sum_{j\in B}\frac{\partial \loss}{\partial \xhatj} \xhatj\xhati.
\end{equation}
\begin{proof} Based on chain rule, we will have
\begin{equation}
    \begin{split}
        \frac{\partial \loss}{\partial \xi} 
        &= \frac{\partial \loss}{\partial \xhati} \frac{\partial\xhati}{\partial\xi} + \sum_{j\in B}\frac{\partial \loss}{\partial \xhatj} \frac{\partial \xhatj}{\partial \psib^2}\frac{\partial\psib^2}{\partial \xi} \\ 
        &= \frac{\partial \loss}{\partial \xhati} \frac{\partial\xhati}{\partial\xi} + \sum_{j\in B}\frac{\partial \loss}{\partial \xhatj} (-\frac{1}{2}\frac{\xj}{\psib^3})\frac{2\xi}{B} \\
        &= \frac{1}{\psib}\frac{\partial \loss}{\partial \xhati} - \frac{1}{B\psib}\sum_{j\in B}\frac{\partial \loss}{\partial \xhatj} \xhatj\xhati.
    \end{split}
\end{equation}
    
\end{proof}
\end{lemma}

Replacing $\frac{\partial \loss}{\partial \xhat}$ by $\gamma \odot \frac{\partial \loss}{\partial \mY}$, we can get~\eref{eq_m:gradient_loss_x_in_pnv}. 

In order to show the effect of \OURSV on the Lipschitz constant of the loss, we make the following standard assumption, as in~\cite{santurkar2018does}. 
\begin{assumption}
\label{ass:assumption_for_lipschitz_constant_of_pnv}
Denote the loss of the non-normalized neural network, which has the same architecture as the \OURSV normalized neural network, as $\losshat$. We assume that
\begin{equation}
    \frac{\partial \loss}{\partial \yi} = \frac{\partial \losshat}{\partial \xi},
\end{equation}
where $\yi$ is the i-th row of $\mY$. 
\end{assumption}

Based on these results, we have the following proof of Lemma~\ref{lemma:lipschitz_constant_of_pnv}, which was stated in Section~\ref{sec:method}.

\begin{proof}[Proof of Lemma~\ref{lemma:lipschitz_constant_of_pnv}] Since all the computational operator of the derivative is element-wise, here we consider $d=1$ for notational simplicity\footnote{For $d\geq 2$, we just need to separate the entry and prove them individually.}. 
When $d=1$, Lemma~\ref{lemma:oursv_gradient_x} can be written as 
\begin{equation}
    \frac{\partial \loss}{\partial \xi} = \frac{1}{\psib}\frac{\partial \loss}{\partial \xhati} - \frac{1}{B\psib}\langle\frac{\partial \loss}{\partial \xhat}, \xhat\rangle \xhati.
\end{equation}
Therefore, we have 
\begin{equation}
    \frac{\partial \loss}{\partial \mX} = \frac{1}{\psib}\frac{\partial \loss}{\partial \xhat} - \frac{1}{B\psib}\langle\frac{\partial \loss}{\partial \xhat}, \xhat\rangle \xhat. 
\end{equation}
Since 
\begin{equation}
    \|\xhat\|^2 = \frac{\sum_{i \in B} \xhati}{\frac{1}{B}\sum_{i \in B} \xhati} = B,
\end{equation}
 the following equation can be obtained
\begin{equation}
\begin{split}
    \|\frac{\partial \loss}{\partial \mX}\|^2 
    & = \frac{1}{\psib^2}\|\frac{\partial \loss}{\partial \xhat} - \langle\frac{\partial \loss}{\partial \xhat}, \frac{\xhat}{\sqrt B}\rangle \frac{\xhat}{\sqrt B}\|^2 \\
    & = \frac{1}{\psib^2}\big(\|\frac{\partial \loss}{\partial \xhat}\|^2 - 2 \langle\frac{\partial \loss}{\partial \xhat},  \langle\frac{\partial \loss}{\partial \xhat}, \frac{\xhat}{\sqrt B}\rangle \frac{\xhat}{\sqrt B}\rangle + \|\langle\frac{\partial \loss}{\partial \xhat}, \frac{\xhat}{\sqrt B}\rangle \frac{\xhat}{\sqrt B}\rangle\|^2\big) \\
    & = \frac{1}{\psib^2}\big(\|\frac{\partial \loss}{\partial \xhat}\|^2 -  \langle\frac{\partial \loss}{\partial \xhat}, \frac{\xhat}{\sqrt B}\rangle^2\big) \\
    & = \frac{\gamma^2}{\psib^2}\big(\|\frac{\partial \loss}{\partial \mY}\|^2 -  \langle\frac{\partial \loss}{\partial \mY}, \frac{\xhat}{\sqrt B}\rangle^2\big) \\ 
    & = \frac{\gamma^2}{\psib^2}\big(\|\frac{\partial \losshat}{\partial \mX}\|^2 -  \langle\frac{\partial \losshat}{\partial \mX}, \frac{\xhat}{\sqrt B}\rangle^2\big).
\end{split}
\end{equation}
\end{proof}

\subsection{Proof of \OURS}
\label{sec:proof_pn_all}

In order to prove that after the replacement of $\frac{\partial \loss}{\partial (\mX^{(t)})}$ with~\eref{eq:input_gradient_replace},
the gradient of the input is bounded, we need the following assumptions.
\begin{assumption}\label{ass:xhati_derivative_bound}
We assume that 
\begin{equation}
    \|\xhati\| \leq C_1~~~\text{and}~~~\|\frac{\partial \loss}{\partial \xhati}\| \leq C_2,
\end{equation}
for all input datum point and all iterations. 
We also assume that the exponentially decaying average of each element of $\xhati$ is bounded away from zero,
\begin{equation}
    (1-\alpha)\sum_{j=0}^t \alpha^{t-j} \xhati\xhati > C_3 > 0,  {~~\forall t},
\end{equation}
where we denote $\alpha$ as the decay factor for the backward pass.
In addition, we assume that $\alpha$ satisfies
\begin{equation}
    (C_1)^2 < \frac{1}{1-\alpha}. 
\end{equation}
\end{assumption}
W.l.o.g., we further assume that every entry of $\psi^{(t)}$ is bounded below, i.e.
\begin{align}
    C_0 < \psi^{(t)},~~~\forall t.
\end{align}

If we can prove or $\nu^{(t)}$ is bounded by some constant $C_4$ (the official proof is in Lemma~\ref{lemma:final_lemma}), then it is obvious to prove the each datum point of $\xtilde'$ is bounded. 

Based on these results, we have the following proof of Theorem~\ref{thm:gradient_xtilde_bound}, which was stated in Section~\ref{sec:method}.

\begin{proof}[Proof of Theorem~\ref{thm:gradient_xtilde_bound}]
It is easy to see that
\begin{align*}
    \|\xtilde'_{i,:}\|^2 
    &= \|\frac{\partial \loss}{\partial \xhati^{(t)}} - \nu^{(t-1)}\xhati^{(t)}\|^2 \\
    &= \langle\frac{\partial \loss}{\partial \xhati^{(t)}} - \nu^{(t-1)}\xhati^{(t)}, \frac{\partial \loss}{\partial \xhati^{(t)}} - \nu^{(t-1)}\xhati^{(t)}\rangle \\
    &= \|\frac{\partial \loss}{\partial \xhati^{(t)}}\|^2 + \|\nu^{(t-1)}\xhati^{(t)}\|^2 -2\langle\frac{\partial \loss}{\partial \xhati^{(t)}}, \nu^{(t-1)}\xhati^{(t)}\rangle \\ 
    &\leq \|\frac{\partial \loss}{\partial \xhati^{(t)}}\|^2 + \|\nu^{(t-1)}\|^2\|\xhati^{(t)}\|^2 -2\langle\frac{\partial \loss}{\partial \xhati^{(t)}}, \nu^{(t-1)}\xhati^{(t)}\rangle \\
    &\leq \|\frac{\partial \loss}{\partial \xhati^{(t)}}\|^2 + \|\nu^{(t-1)}\|^2\|\xhati^{(t)}\|^2 + 2\|\frac{\partial \loss}{\partial \xhati^{(t)}}\| \|\nu^{(t-1)}\xhati^{(t)}\| \\
    &\leq \|\frac{\partial \loss}{\partial \xhati^{(t)}}\|^2 + \|\nu^{(t-1)}\|^2\|\xhati^{(t)}\|^2 + 2\|\frac{\partial \loss}{\partial \xhati^{(t)}}\| \|\nu^{(t-1)}\|\|\xhati^{(t)}\| \\
    & \leq (C_2)^2 + (C_1)^2 (C_4)^3 + C_1C_2C_4 \\
\end{align*}
All these inequalities come from Cauchy-Schwarz inequity and the fact that 
\[
(a_1b_1)^2 + ... + (a_db_d)^2 \leq (a_1^2 + ... + a_d^2)(b_1^2 + ... + b_d^2). 
\]
\end{proof}

In the final step of Theorem~\ref{thm:gradient_xtilde_bound}, we directly use that $\nu^{(t)}$ is uniformly bounded (each element of $\nu^{(t)}$ is bounded) by $C_4$. 
The exact proof is shown in below. 

\begin{lemma}
\label{lemma:final_lemma}
Under Assumption~\ref{ass:xhati_derivative_bound}, $\nu^{(t)}$ is uniformly bounded.
\begin{proof}
For simplicity, denote $\frac{\partial \loss}{\partial \xhati}$ as $\xhati'$. It is not hard to see, 
\begin{align*}
    \|\Gamma^{(t)}\|^2
    &= \frac{1}{B^2}\|\sum_{i=1}^B\xhati^{(t)}\xhati^{(t)}\|^2 \\
    &= \frac{1}{B^2}\langle\sum_{i=1}^B\xhati^{(t)}\xhati^{(t)}, \sum_{i=1}^B\xhati^{(t)}\xhati^{(t)}\rangle \\
    &\leq \frac{1}{B^2}(B^2 \max_{j}\{\langle\xhatj^{(t)}\xhatj^{(t)}, \xhati^{(t)}\xhati^{(t)}\rangle\}) \\ 
    &\leq (C_1)^2.
\end{align*}
Similarly, we will have $\|\Lambda^{(t)}\|\leq C_1C_2$ as well as $(1-\alpha)\sum_{j=0}^t \alpha^{t-j} \Gamma^{(j)}\Gamma^{(j)} > C_3$. 
We have 
\begin{align*}
    \nu^{(t)} 
    &= (1-(1-\alpha)\Gamma^{(t)})\nu^{(t-1)} + (1-\alpha)\Lambda^{(t)} \\ 
    &= (1-(1-\alpha)\Gamma^{(t)})((1-(1-\alpha)\Gamma^{(t-1)})\nu^{(t-2)} + (1-\alpha)\Lambda^{(t-1)}) + (1-\alpha)\Lambda^{(t)} \\ 
    & \vdots \\ 
    & = (1-\alpha) \sum_{j=0}^t \big( \prod_{k=0}^{j-1}(1-(1-\alpha)\Gamma^{(t-k+1)})\big)\Lambda^{(t-j)}.
\end{align*}
Then, 
\begin{align*}
    \frac{1}{(1-\alpha)^2}\|\nu^{(t)}\|^2
    =& \langle \sum_{j=0}^t \big( \prod_{k=0}^{j-1}(1-(1-\alpha)\Gamma^{(t-k+1)})\big)\Lambda^{(t-j)}, \sum_{j=0}^t \big( \prod_{k=0}^{j-1}(1-(1-\alpha)\Gamma^{(t-k+1)})\big)\Lambda^{(t-j)}\rangle. \\
\end{align*}
Notice that with the definition,
\begin{equation}
\Gamma^{(m)}= \frac{1}{B}\sum_{i=1}^B\xhati^{(m)}\xhati^{(m)},
\end{equation}
we will have that all entries of $\Gamma^{(m)}$ are positive, for $m\in \{0, 1,...,t\}$. 
It is clear that when all entries of $\Lambda^{(m)}$, for $m\in \{0, 1,...,t\}$, have the same sign (positive or negative), the above equation achieves its upper bound. W.l.o.g., we assume they are all positive. 

Since $0<\alpha<1$, it is easy to see that, when $K=\ceil{(\log({\frac{(1-\alpha)C_3}{2C_1C_1}})/\log(\alpha))}$, then the following inequality holds,
\begin{equation}\label{eq:alpha_constant}
    (1-\alpha)\sum_{j=K}^\infty \alpha^j < \frac{C_3}{2C_1C_1}.
\end{equation}
Since $\|\Gamma^{(k)}\| \leq C_1$, the value of any entry of $\Gamma^{(k)}$ is also bounded by $C_1$.
Therefore, based on this and \eref{eq:alpha_constant}, when $t>K$, we will have 
\begin{equation}
\begin{split}
    (1-\alpha)\sum_{k=t-K+1}^t\alpha^{t-k}~\Gamma^{(k)}\Gamma^{(k)}
    &= (1-\alpha) \sum_{k=0}^t \alpha^{t-k}~\Gamma^{(k)}\Gamma^{(k)} - (1-\alpha) \sum_{k=0}^{t-K} \alpha^{t-k}~\Gamma^{(k)}\Gamma^{(k)} \\
    &> C_3\1 - (1-\alpha) \sum_{k=0}^{t-K} \alpha^{t-k}~\|\Gamma^{(k)}\|\|\Gamma^{(k)}\| \\
    & > C_3\1 - (1-\alpha)C_1C_1 \1 \sum_{k=0}^{t-K} \alpha^{t-k} \\
    & = C_3\1 - (1-\alpha)C_1C_1 \1 \sum_{k=K}^{t} \alpha^k \\
    &>  C_3\1 - (1-\alpha)C_1C_1 \1 \sum_{k=K}^{\infty} \alpha^k \\ 
    & > C_3\1 - \frac{C_3\1}{2} \\ 
    &=\frac{C_3\1}{2}, \label{eq:bound_C5}
\end{split}
\end{equation}
where $\1$ is the unit vector. 
Then, for $t>K$, we can bound from below the arithmetic average of the $K$ corresponding items of $\Gamma$,
\begin{equation}
\begin{split}
    (C_1)^2 > \frac{1}{K}\sum_{k=0}^{K-1} \Gamma^{(t-k)}\Gamma^{(t-k)} 
    &> \frac{1}{\alpha^{K-1}}\sum_{k=0}^{K-1}\alpha^k\Gamma^{(t-k)}\Gamma^{(t-k)} \\
    &=\frac{1}{\alpha^{K-1}}\sum_{k=t-K+1}^{t}\alpha^{t-1} \Gamma^{(k)}\Gamma^{(k)} \\ 
    &> \frac{C_3}{2(1-\alpha)\alpha^{K-1}} = C_5 > 0.
\end{split}
\end{equation}
This inequality shows that after the first K items, for any K consecutive $\Gamma^{(k)}$, the average of them will exceeds a constant number, $C_5$. 
Therefore, for any $t> T > K$, we will have
\begin{equation}\label{eq:bound_consective_K}
    \frac{1}{T-K}\sum_{k=0}^{T-K}\Gamma^{(t-k)}\Gamma^{(t-k)} > \floor{\frac{T-K}{K}}(K\frac{1}{T-K})C_5 > \frac{C_5}{2}.
\end{equation}

Let us split $\sum_{j=0}^t \big( \prod_{k=0}^{j-1}(1-(1-\alpha)\Gamma^{(t-k+1)})\big)\Lambda^{(t-j)}$ into two parts: (i) $\sum_{j=K}^{t} \big( \prod_{k=0}^{j-1}(1-(1-\alpha)\Gamma^{(t-k+1)})\big)\Lambda^{(t-j)}$, and (ii) $\sum_{j=0}^{K-1} \big( \prod_{k=0}^{j-1}(1-(1-\alpha)\Gamma^{(t-k+1)})\big)\Lambda^{(t-j)}$. 
From so on, we will discuss how we deal with these two parts respectively.

\paragraph{Case 1: $\sum_{j=K}^{t} \big( \prod_{k=0}^{j-1}(1-(1-\alpha)\Gamma^{(t-k+1)})\big)\Lambda^{(t-j)}$}
Notice that for $0< a_j < 1$, the following inequality can be proven with simply induction,
\begin{equation}
    \prod_{j=0}^{k-1} (1-a_j) \leq (1-\frac{1}{k}\sum_{j=0}^{k-1}\alpha_j)^k. 
\end{equation}
Replacing $a_j$ with $(1-\alpha)\Gamma^{(t-j+1)}$, we will have 
\begin{equation}
    \begin{split}
        \sum_{j=K}^{t} \big( \prod_{k=0}^{j-1}(1-(1-\alpha)\Gamma^{(t-k+1)})\big)\Lambda^{(t-j)} 
        &\leq \sum_{j=K}^{t} \big( (1-\frac{(1-\alpha)}{j}\sum_{k=0}^{j-1}\Gamma^{(t-k+1)})\big)^j\Lambda^{(t-j)} \\
        &\leq \sum_{j=K}^{t} \big( (1-(1-\alpha)\frac{C_5}{2})\big)^j\Lambda^{(t-j)} \\ 
        &\leq \sum_{j=K}^{t} \big( (1-(1-\alpha)\frac{C_5}{2})\big)^jC_1C_2 \\ 
        & \leq \frac{2}{(1-\alpha)C_5}C_1C_2 = C_6.
    \end{split}
\end{equation}
Here the second inequality comes from~\eref{eq:bound_consective_K}, and the third inequality comes form the the fact each entry of $\Lambda^{(m)}$ is smaller than $C_1C_2$, given $\|\Lambda^{(m)}\|\leq C_1C_2$. The final inequality comes from~\eref{eq:bound_C5}, where $0 < C_5 < (C_1)^2 < 1/(1-\alpha)$, then we can have $0 < \big(1-(1-\alpha)C_5/{2}\big)<1$.

\paragraph{Case 2: $\sum_{j=0}^{K-1} \big( \prod_{k=0}^{j-1}(1-(1-\alpha)\Gamma^{(t-k+1)})\big)\Lambda^{(t-j)}$}
It is easy to see
\begin{equation}
    \begin{split}
        \sum_{j=0}^{K-1} \big( \prod_{k=0}^{j-1}(1-(1-\alpha)\Gamma^{(t-k+1)})\big)\Lambda^{(t-j)} 
        & \leq \sum_{j=0}^{K-1} \big( \prod_{k=0}^{j-1}(\1)\big)\Lambda^{(t-j)} \\
        & \leq KC_1C_2.
    \end{split}
\end{equation}

Combining Case 1 and 2, we have
\begin{equation}
    \begin{split}
        \frac{1}{(1-\alpha)^2}\|\nu^{(t)}\|^2
    &= \langle \sum_{j=0}^t \big( \prod_{k=0}^{j-1}(1-(1-\alpha)\Gamma^{(t-k+1)})\big)\Lambda^{(t-j)}, \sum_{j=0}^t \big( \prod_{k=0}^{j-1}(1-(1-\alpha)\Gamma^{(t-k+1)})\big)\Lambda^{(t-j)}\rangle \\
    &\leq \langle C_6\1 + KC_1C_2\1, C_6\1 + KC_1C_2\1 \rangle < C_7, 
    \end{split}
\end{equation}
which indicates $\|\nu^{(t)}\|$ is bounded and $C_4 = (1-\alpha)\sqrt{C_7}$. 

\end{proof}

\end{lemma}

\end{document}